%% file: main.tex
\documentclass[english,ruled]{article}
\usepackage[pdftex]{graphicx}
\usepackage[T1]{fontenc}
\usepackage[latin9]{inputenc}
\usepackage{verbatim}
\usepackage{algorithm2e}
\usepackage{amsmath}
\usepackage{amsthm}
\usepackage{amssymb}
\usepackage{bm}
\usepackage{xcolor}
\usepackage{multicol}
\usepackage{multirow}
\usepackage{geometry}
\usepackage{booktabs}
\usepackage{enumitem}
\usepackage{setspace}
\makeatletter
\usepackage[toc,page,header]{appendix}
\usepackage{minitoc}
\usepackage{ifthen}
\newboolean{doublecolumn}
\setboolean{doublecolumn}{false}
\newboolean{arxiv}
\setboolean{arxiv}{true}
\usepackage{pgfplots}
\usepackage{pdflscape}
\usepackage{hyperref}
\usepackage{cleveref}
\usepackage{authblk}
\pgfplotsset{compat=1.15}
\input{macros.tex}

\renewcommand \thepart{}
\renewcommand \partname{}

\theoremstyle{plain}
\newtheorem{lem}{\protect\lemmaname}[section]
\theoremstyle{remark}
\newtheorem{rem}{\protect\remarkname}[section]
\theoremstyle{plain}
\newtheorem{thm}{\protect\theoremname}[section]
\theoremstyle{plain}
\newtheorem{prop}{\protect\propositionname}[section]
\theoremstyle{plain}
\newtheorem{exple}{\protect\examplename}[section]
\providecommand{\corollaryname}{Corollary}
\theoremstyle{plain}
\newtheorem{coro}{\protect\corollaryname}[section]

\providecommand{\propositionname}{Proposition}
\providecommand{\lemmaname}{Lemma}
\providecommand{\remarkname}{Remark}
\providecommand{\theoremname}{Theorem}
\providecommand{\examplename}{Example}

\newcommand{\assign}{:=}
\newcommand{\tmtextbf}[1]{\text{{\bfseries{#1}}}}
\newcommand{\tmop}[1]{\ensuremath{\operatorname{#1}}}

\newcommand{\sgm}{{\texttt{SGM}}}

\newcommand{\x}{\mathbf{x}}
\newcommand{\tmc}{\mathbf{c}}
\newcommand{\p}{\mathbf{p}}
\newcommand{\s}{\mathbf{s}}

\newcommand{\tmd}{\mathbf{d}}
\newcommand{\A}{\mathbf{A}}
\newcommand{\tmb}{\mathbf{b}}
\newcommand{\1}{\mathbf{1}}
\newcommand{\y}{\mathbf{y}}
\newcommand{\tma}{\mathbf{a}}
\newcommand{\ld}{\underline{d}}
\newcommand{\g}{\mathbf{g}}
\newcommand{\z}{\mathbf{z}}
\newcommand{\I}{\mathbf{I}}
\newcommand{\tmv}{\mathbf{v}}
\newcommand{\bc}{\bar{c}}
\newcommand{\ba}{\bar{a}}
\newcommand{\bd}{\bar{d}}
\newcommand{\bld}{\underline{d}}
\newcommand{\ud}{\bar{d}}
\newcommand{\ua}{\bar{a}}
\newcommand{\uc}{\bar{c}}
\newcommand{\mathd}{\mathrm{d}}

\newcommand{\goodnewline}{\\}

\newcommand{\condsmall}{\ifthenelse{\boolean{doublecolumn}}{\small}{}}

\crefdefaultlabelformat{#2\textbf{#1}#3} % <-- Only #1 in \textbf
\crefname{section}{\textbf{section}}{\textbf{sections}}
\Crefname{section}{\textbf{Section}}{\textbf{Sections}}
\crefname{thm}{\textbf{theorem}}{\textbf{theorems}}
\Crefname{thm}{\textbf{Theorem}}{\textbf{Theorems}}
\crefname{lem}{\textbf{lemma}}{\textbf{lemmas}}
\Crefname{lem}{\textbf{Lemma}}{\textbf{Lemmas}}
\crefname{prop}{\textbf{proposition}}{\textbf{propositions}}
\Crefname{prop}{\textbf{Proposition}}{\textbf{Propositions}}
\crefname{algorithm}{\textbf{algorithm}}{\textbf{algorithms}}
\Crefname{algorithm}{\textbf{Algorithm}}{\textbf{Algorithms}}
\crefname{coro}{\textbf{Corollary}}{\textbf{corollaries}}
\Crefname{coro}{\textbf{Corollary}}{\textbf{corollaries}}
\crefname{definition}{\textbf{Definition}}{\textbf{definitions}}
\Crefname{definition}{\textbf{Definition}}{\textbf{definitions}}
\crefname{table}{\textbf{Table}}{\textbf{tables}}
\Crefname{table}{\textbf{Table}}{\textbf{tables}}
\crefname{figure}{\textbf{Figure}}{\textbf{figures}}
\Crefname{figure}{\textbf{Figure}}{\textbf{figures}}

\makeatother

\author[1]{Wenzhi Gao\thanks{\url{gwz@stanford.edu}, this paper is an extended version of \cite{gao2024decoupling}}}
\author[2]{Dongdong Ge\thanks{\url{dongdong@gmail.com}, corresponding author}}
\author[1]{Chunlin Sun\thanks{\url{chunlin@stanford.edu}}}
\author[3]{Chenyu Xue\thanks{\url{xcy2721d@gmail.com}, corresponding author}}
\author[1,4]{Yinyu Ye\thanks{\url{yyye@stanford.edu}}}
\affil[1]{ICME, Stanford University}
\affil[2]{Antai College of Economics and Management, Shanghai Jiao Tong University}
\affil[3]{RIIS, Shanghai University of Finance and 
Economics}
\affil[4]{Management Science and Engineering, Stanford University}

\title{Beyond $\Ocal(\sqrt{T})$ Regret: Decoupling Learning and Decision-making in Online Linear Programming}

\begin{document}

\maketitle
\input{abstract.tex}

\input{sec_intro.tex}
\input{sec_setup.tex}
\input{sec_errorbnd.tex}
\input{sec_algorithm.tex}
\input{sec_exp.tex}
\input{sec_conclusion.tex}

\renewcommand \thepart{}
\renewcommand \partname{}

%\newpage
\bibliographystyle{plain}
\bibliography{olp.bib}

\doparttoc
\faketableofcontents
\part{}

\appendix
\onecolumn
\addcontentsline{toc}{section}{Appendix}
\part{Appendix} 
\parttoc

\vspace{10pt}

\newpage

\input{app_errbnd.tex}
\input{app_framework.tex}

\end{document}

%% file: macros.tex
\global\long\def\vertiii#1{\left\vert \kern-0.25ex  \left\vert \kern-0.25ex  \left\vert #1\right\vert \kern-0.25ex  \right\vert \kern-0.25ex  \right\vert }%

\global\long\def\diam{\mathrm{diam}}%

\global\long\def\inte{\operatornamewithlimits{int}}%

\global\long\def\argmin{\operatornamewithlimits{arg\,min}}%

\global\long\def\and{\mathrm{and}}%

\global\long\def\dist{\mathrm{dist}}%

\global\long\def\Ebb{\mathbb{E}}%

\global\long\def\Rbb{\mathbb{R}}%

\global\long\def\Acal{\mathcal{A}}%

\global\long\def\Ocal{\mathcal{O}}%

\global\long\def\Ycal{\mathcal{Y}}%

%% file: abstract.tex
\begin{abstract}
Online linear programming plays an important role in both revenue management
  and resource allocation, and recent research has focused on developing
  efficient first-order online learning algorithms. Despite the empirical
  success of first-order methods, they typically achieve a regret no better
  than $\mathcal{O} ( \sqrt{T} )$, which is suboptimal compared to
  the $\mathcal{O} (\log T)$ bound guaranteed by the state-of-the-art linear
  programming (LP)-based online algorithms. This paper establishes a general
  framework that improves upon the $\mathcal{O} ( \sqrt{T} )$
  result when the LP dual problem exhibits certain error bound conditions. For the first time, we show that first-order learning algorithms achieve $o( \sqrt{T} )$ regret in the
  continuous support setting and $\mathcal{O} (\log T)$ regret in the finite support setting beyond the non-degeneracy assumption. Our results significantly improve the state-of-the-art regret results and provide new insights for sequential decision-making.
\end{abstract}

%% file: sec_intro.tex
\section{Introduction} \label{sec:intro}

This paper presents a new algorithmic framework to solve the online linear programming (OLP) problem. In this context, a decision-maker receives a sequence of resource requests with bidding prices and sequentially makes irrevocable allocation decisions for these requests. 
OLP aims to maximize the accumulated reward subject to inventory or resource constraints.  OLP plays an important role in a wide range of applications. For example, in online advertising {\cite{10.48550/arxiv.2011.10124}}, an online platform has limited advertising slots on a web page. When a web page loads, online advertisers bid for ad placement, and the platform decides within milliseconds the slot allocation based on the features of advertisers and the user. The goal is to maximize the website's revenue and improve user experience. Another example is online auction, where an online auction platform hosts a large number of auctions for different items. The platform must handle bids and update the auction status in real-time. Besides the aforementioned applications, OLP is also widely used in applications such as revenue management \cite{talluri2004theory}, resource allocation \cite{katoh1998resource}, cloud computing \cite{hussain2013survey}, and many other applications \cite{balseiro2023survey}. \goodnewline

Most state-of-the-art algorithms for OLP are dual linear program (LP)-based
{\cite{agrawal2014dynamic,kesselheim2014primal,ma2024optimal,li2022online,ma2024optimal}}. More specifically, these algorithms require solving a sequence of LPs to make online decisions.  However, the high computational cost of these LP-based methods prevents their application in time-sensitive or large-scale problems. For example, in the aforementioned online advertising example, a decision has to be made in milliseconds, while LP-based methods can take minutes to hours on large-scale problems. This challenge motivates a recent line of research
using first-order methods to address OLP
{\cite{10.48550/arxiv.2003.02513,gao2023solving,10.48550/arxiv.2011.10124,balseiro2022online}}, which are based on gradient information and more scalable and computationally efficient than LP-based methods. \goodnewline

Despite the advantage in computational efficiency, first-order methods are still not comparable to LP-based methods in terms of regret for many settings. Existing first-order OLP algorithms only achieve $\mathcal{O} ( \sqrt{T} )$ regret bound with only a few exceptions. When the distribution of requests and bidding prices has finite support \cite{sun2020near}, $\mathcal{O} ( T^{3/8} )$ regret is obtainable using a three-stage algorithm; if first-order methods are used to solve the subproblems of LP-based methods infrequently, $\mathcal{O} ( \log^2 T )$ regret is achievable in the continuous support setting under a uniform non-degeneracy assumption \cite{ma2024optimal}. However, these methods are either complicated to implement or require strong assumptions. It remains open whether there exists a general framework that allows first-order methods to break the $\mathcal{O} ( \sqrt{T} )$ regret barrier. This paper answers this question affirmatively.
\subsection{Contributions}
\begin{itemize}[leftmargin=10pt]
  \item We show that first-order methods achieve  $o
  ( \sqrt{T} )$ regret under weaker assumptions than LP-based methods. In particular, we identify a dual error bound condition that is sufficient to guarantee lower regret of first-order methods when the dual LP problem has a unique optimal solution. In the continuous-support setting, we establish an $\mathcal{O} (T^{1/3})$ regret result under weaker assumptions than the existing methods; In the finite-support setting, we establish an $\mathcal{O} (\log T)$ result, which significantly improves on the state-of-the-art $\mathcal{O} (T^{3/8})$ result and almost matches the $\mathcal{O} (1)$ regret of LP-based methods. For problems with $\gamma$-H\"{o}lder growth condition, we establish a general ${\mathcal{O} (T^{\frac{\gamma-1}{2\gamma - 1}}\log T)}$ regret result, which interpolates among the settings of no growth ($\gamma = \infty$), continuous support with quadratic growth ($\gamma = 2$) and finite support with sharpness ($\gamma = 1$). Our results show that first-order methods perform well under strictly weaker conditions than LP-based methods and still achieve competitive performance, significantly advancing their applicability in practice.
  \item We design a general exploration-exploitation framework to exploit the dual error bound condition. The idea is to learn a good approximation of the distribution dual optimal solution. Then, the online decision-making algorithm can be localized around a neighborhood of the approximate dual solution and makes decisions in an effective domain of size $o(1)$, thereby achieving improved regret guarantees. We reveal an important dilemma in simultaneously using a single first-order method as both learning and decision-making algorithms: a good learning algorithm can perform poorly in decision-making. This dilemma implies an important discrepancy between stochastic optimization and online decision-making, and it is addressed by decoupling learning and decision-making: two different first-order methods are adopted for learning and decision-making. This simple idea yields a highly flexible framework for online sequential decision-making. Our analysis can be of independent interest in the broader context of online convex optimization.
\end{itemize}

\begin{table*}[h]
\centering
  \caption{Regret results in the current OLP literature. $\log \log$ factors are ignored.\label{tab:compare}}
\resizebox{\textwidth}{!}{
  \begin{tabular}{cccll}
    \toprule
    Paper & Setting and assumptions& Algorithm & Regret &  Lower bound \\
    \midrule
    {\cite{li2022online}} & Bounded, continuous support, uniform non-degeneracy & LP-based & $\mathcal{O}
    (\log T)$ & Yes\\
    {\cite{bray2019logarithmic}} & Bounded, continuous support, uniform non-degeneracy & LP-based &
    $\mathcal{O} (\log T)$ & Yes \\
    {\cite{10.48550/arxiv.2210.07996}} & Bounded, finite support of $\tma_t$, quadratic growth & LP-based &
    $\mathcal{O} (\log^2 T)$ & Unknown \\
    {\cite{ma2024optimal}} & Bounded, continuous support, uniform non-degeneracy & LP-based &
    $\mathcal{O} (\log T)$ & Yes \\
   	    \cite{10.48550/arxiv.2101.11092} & Bounded, finite support, non-degeneracy &  LP-based & $\mathcal{O}(1)$ & Yes \\
   	    \midrule
    {\cite{10.48550/arxiv.2003.02513}} & Bounded & Subgradient &
    $\mathcal{O} ( \sqrt{T} )$ & Yes \\
    {\cite{10.48550/arxiv.2011.10124}} & Bounded & Mirror Descent &
    $\mathcal{O} ( \sqrt{T} )$ & Yes \\
    {\cite{gao2023solving}} & Bounded &  Proximal Point & $\mathcal{O}
    ( \sqrt{T} )$ & Yes \\
    {\cite{balseiro2022online}} & Bounded & Momentum & $\mathcal{O}
    ( \sqrt{T} )$ & Yes \\
    {\cite{sun2020near}} & Bounded, finite support, non-degeneracy &  Subgradient &
    $\mathcal{O} (T^{3 / 8})$ & No ($\mathcal{O}(1)$)\\
        {\cite{ma2024optimal}} & Bounded, continuous support, uniform non-degeneracy &  Subgradient &
    $\mathcal{O} (\log^2 T)$ & No ($\mathcal{O}(1)$)\\
\midrule
        This paper & Bounded, continuous support, quadratic growth & Subgradient & ${\mathcal{O} (T^{1
    / 3})}$ & No ($\mathcal{O}(\log T)$)\\
    This paper & Bounded, finite support, sharpness &  Subgradient & ${\mathcal{O} (\log T)}$ & No ($\mathcal{O}(1)$)\\
        This paper & Bounded, $\gamma$-dual error bound, unique solution &  Subgradient & ${\mathcal{O} (T^{\frac{\gamma - 1}{2\gamma -1}}\log T)}$ & Unknown\\
    \bottomrule
  \end{tabular}
  }
\end{table*}
\paragraph{Related Literature.} There is a vast amount of literature on OLP \cite{ma2020algorithms,mirrokni2012simultaneous,mahdian2012online,arlotto2019uniformly}, and we review some recent developments that go beyond $\mathcal{O} ( \sqrt{T} )$ regret in the stochastic input setting (\textbf{Table \ref{tab:compare}}). These algorithms mostly follow the same principle of making decisions based on the learned information: learning and decision-making are closely coupled with each other. We refer the interested readers to \cite{balseiro2023survey} for a more detailed review of OLP and relevant problems.

\paragraph{LP-based OLP Algorithms.} 
Most LP-based methods leverage the dual LP problem {\cite{agrawal2014dynamic}}, with only a few exceptions
{\cite{kesselheim2014primal}}. Under assumptions of either non-degeneracy or
finite support on resource requests and/or rewards, $\mathcal{O}
(\log T)$ regret has been achieved under different settings. More specifically,
{\cite{li2022online}} establish the dual convergence of finite-horizon
LP solution to the optimal dual solution to the underlying stochastic program. In the continuous support setting, $\mathcal{O} (\log T \log \log T)$ regret is achievable. {\cite{bray2019logarithmic}} considers multi-secretary problem and
establishes an $\mathcal{O} (\log T)$ regret result.
{\cite{ma2024optimal}} consider the setting where a
regularization term is imposed on the resource and also establish an
$\mathcal{O} (\log T)$ regret result. {\cite{10.48550/arxiv.2210.07996}} establish $\mathcal{O} (\log^2 T)$ regret without the non-degeneracy assumption and assume that the distribution of resource requests has finite support. {\cite{10.48550/arxiv.2101.11092}} consider the case where both resource requests and prices have finite support and $\Ocal(1)$ regret can be achieved in this case under a non-degeneracy assumption. Recently, attempts have been made to address the computation cost of LP-based methods by infrequently solving the LP subproblems \cite{li2024infrequent,xu2024online, sun2024wait}. Most LP-based methods follow the action-history dependent approach developed in \cite{li2022online} to achieve $o(\sqrt{T})$ regret, and in the continuous support case, the non-degeneracy assumption is required to hold uniformly for resource vector $\tmb$ in some pre-specified region. 
Compared to the aforementioned LP-based methods, our framework can work under strictly weaker assumptions.

\paragraph{First-order OLP Algorithms.} Early explorations of first-order OLP algorithms start from \cite{10.48550/arxiv.2003.02513}, \cite{10.48550/arxiv.2011.10124} and \cite{lobos2021joint}, where $\mathcal{O}(\sqrt{T})$ regret is established using mirror descent and subgradient methods. \cite{gao2023solving} show that proximal point update also achieves $\mathcal{O}(\sqrt{T})$ regret. \cite{balseiro2022online} analyze a momentum variant of mirror descent and get $\mathcal{O}(\sqrt{T})$ regret. In the finite support setting, \cite{sun2020near} design a three-stage algorithm that achieves $\mathcal{O}(T^{3/8})$ regret when the distribution LP is non-degenerate. \cite{ma2024optimal} apply a first-order method to solve subproblems in LP-based methods infrequently and achieve $\mathcal{O}(\log^2 T)$ regret. However, \cite{ma2024optimal} still requires a uniform non-degeneracy assumption. Our framework is motivated directly by the properties of first-order methods and provides a unified analysis under different distribution settings.

\paragraph{Structure of the paper}The rest of the paper is organized as
follows. \Cref{sec:first-order} introduces OLP and first-order OLP algorithms.
\Cref{sec:errbnd} defines the dual error bound condition and
its implications on the first-order learning algorithms. In \Cref{sec:algo}, we introduce a general framework that exploits the error bound condition
and improves the state-of-the-art regret results. We verify the theoretical
findings in \Cref{sec:exp}.

%% file: sec_setup.tex
\section{Online linear programming with first-order methods} \label{sec:first-order}

\paragraph{Notations.} Throughout the paper, we use $\| \cdot \|$ to denote Euclidean norm and $\langle \cdot, \cdot \rangle$ to denote Euclidean inner product. Bold letter notations $\A$ and $\tma$ denote matrices and vectors, respectively. Given a convex function $f ( \x )$, its subdifferential is denoted by $\partial f ( \x ) \assign \{ \tmv : f ( \y ) \geq f ( \x ) + \langle \tmv, \y - \x \rangle, \text{ for all } \y \}$ and $f' ( \x ) \in \partial f ( \x )$ is called a subgradient. We use $\g_\x$ satisfying $\Ebb[\g_\x] \in \partial f(\x)$ to denote a stochastic subgradient. $[\cdot]_+ = \max \{ \cdot, 0 \}$ denotes the component-wise positive-part function, and $\mathbb{I} \{ \cdot \}$ denotes the 0-1 indicator function. Relation $\x \geq \y$ denotes element-wise inequality. Given $\x$ and a closed convex set $\mathcal{X}$, we define ${\dist} ( \x, \mathcal{X} ) \assign \min_{\y \in \mathcal{X}} \| \x - \y \|$ and $\diam (\mathcal{X}) \assign \max_{\x, \y
\in \mathcal{X}}  \| \x - \y \|$.

\subsection{OLP and duality}

An online resource allocation problem with linear inventory and rewards can be modeled
as an OLP problem: given time horizon $T \geq 1$ and $m \geq 1$ resources
represented by $\tmb \in \mathbb{R}^m_+$, at time $t$, a customer with $( c_t,
\tma_t ) \in \mathbb{R} \times \mathbb{R}^m$ arrives and requests
resources $\tma_t$ at price $c_t$. Decision $x^t \in [0, 1]$ is made to
(partially) accept the order or reject it. With compact notation $\tmc =
(c_1, \ldots, c_T)^{\top} \in \mathbb{R}^T$ and $\A \assign ( \tma_1,
\ldots, \tma_T ) \in \mathbb{R}^{m \times T}$, the problem can be
written as
\begin{align} \label{olp:primal}
\tag{PLP}
\max_{\x} \quad  \langle \tmc, \x \rangle  \quad\text{subject to} \quad \A \x \leq \tmb, \nonumber
& \quad \mathbf{0} \leq \x \leq \1, \nonumber
\end{align}
where $\mathbf{0}$ and $\1$ are vectors of all zeros and ones. The dual problem 
\begin{align} \label{olp:dual-lp}
\tag{DLP}
 \min_{(\y,\s) \geq \mathbf{0}} \quad \langle \tmb, \y \rangle  + \langle \1, \s \rangle \quad\text{subject to} \quad \s \geq \tmc - \A^\top \y
\end{align}
can be transformed into the following finite sum form
\begin{align} \label{olp:dual-sample}
 \min_{\y \geq \mathbf{0}} ~f_T ( \y ) \assign \textstyle \frac{1}{T} \sum_{t = 1}^T
   \langle \tmd, \y \rangle + [ c_t - \langle \tma_t, \y
   \rangle ]_+,
\end{align}
where $\tmd = T^{- 1} \tmb$ is the average resource. When
$( c_t, \tma_t )$ are i.i.d. distributed, $f_T ( \y )$
can be viewed as a sample approximation of the expected dual function $f
( \y )$, where
\begin{equation}\label{eqn:distribution-dual}
	f ( \y ) \assign \mathbb{E} [ f_T ( \y )
   ] = \langle \tmd, \y \rangle +\mathbb{E}_{( c, \tma
   )} [ c - \langle \tma, \y \rangle ]_+ .
\end{equation}
Define the sets of dual optimal solutions, and let $\y_T^\star, \y^\star$ be some dual optimal solutions, respectively:
\[ \y_T^{\star} \in \mathcal{Y}_T^{\star} = \argmin_{\y \geq \mathbf{0}} ~f_T (
   \y ) \quad \text{and} \quad \y^{\star} \in \mathcal{Y}^{\star} = \argmin_{\y \geq \mathbf{0}} ~f ( \y ), \]
and we can determine the primal optimal solution $\x_T^{\star} = (x_1^{\star}, \ldots,
x_T^{\star}) \in \mathbb{R}^T$ by the LP optimality conditions:
\[ x_t^{\star} \in \left \{ \begin{array}{cl}
     \{ 0 \}, & \text{ if } c_t < \langle \tma_t, \y_T^{\star} \rangle,\\
     {}[0, 1], & \text{ if } c_t = \langle \tma_t, \y_T^{\star} \rangle,\\
     \{ 1 \}, & \text{ if } c_t > \langle \tma_t, \y_T^{\star} \rangle.
   \end{array} \right. \]
This connection between primal and dual solutions motivates dual-based online
learning algorithms: a dual-based learning algorithm maintains a dual sequence $\{ \y^t \}_{t = 1}^T$ in the learning process, while primal
decisions are made based on the optimality condition. Given the sample
approximation interpretation, first-order methods are natural candidate learning algorithms.

\subsection{First-order methods on the dual problem}
First-order methods leverage the
sample approximation structure and applies (sub)gradient-based first-order update. One commonly used first-order method is the online projected subgradient method (\Cref{alg:subgrad}):
\begin{align}
  x^t ={} & \mathbb{I} \{ c_t \geq \langle \tma_t, \y^t \rangle
  \}, \nonumber\\
  \g^t \in{} & \partial_{\y = \y_t} \{ \langle \tmd, \y \rangle
  + [ c_t - \langle \tma_t, \y \rangle ]_+ \},
  \nonumber\\
  \y^{t + 1} ={} & [ \y^t - \alpha_t \g^t ]_+  \label{eqn:osgm}. 
\end{align}
Upon the arrival of each customer $(c_t, \tma_t)$, a decision $x^t$ is made based on the
optimality condition. Then, the dual variable $\y^t$ is adjusted with the stochastic subgradient. Other learning algorithms, such as mirror descent, also apply to the OLP setting. This paper focuses on the subgradient method.
\begin{algorithm}[h]
\caption{First-order subgradient OLP algorithm \label{alg:subgrad}}	
\KwIn{Initial dual solution guess $\y^1$, subgradient stepsize $\{\alpha_t\}_{t=1}^T$}
\For{$t$ = \rm{$1$ to $T$ }}{
Make primal decision $x^t = \mathbb{I} \{ c_t \geq \langle \tma_t, \y^t \rangle \}$\\
Compute subgradient $\g^t = \tmd - \tma_t x^t$\\
Subgradient update $\y^{t + 1} = [ \y^t - \alpha_t \g^t ]_+$
}
\end{algorithm}
\subsection{Performance metric}

Given online algorithm output $\hat{\x}_T = (x^1, \ldots, x^T)$, its regret and constraint violation are defined as
\begin{align}
  r ( \hat{\x}_T ) \assign & \max_{\A \x \leq \tmb, \mathbf{0} \leq \x
  \leq \1}  \langle \tmc, \x \rangle - \langle \tmc,
  \hat{\x}_T \rangle \qquad  \text{and} \qquad 
  v ( \hat{\x}_T ) \assign  \| [ \A \hat{\x}_T -
  \tmb ]_+ \|. \nonumber
\end{align}
These metrics are widely used in the OLP literature \cite{10.48550/arxiv.2003.02513,gao2023solving}.

\subsection{Main assumptions and summary of the results}

We make the following assumptions throughout the paper.
\begin{enumerate}[leftmargin=30pt, label=\textbf{A\arabic*.},ref=\rm{\textbf{A\arabic*}}]
  \item (Stochastic input) $\{ ( c_t, \tma_t ) \}_{t=1}^T$ are generated i.i.d.
  from some distribution $\mathcal{P}$. \label{A1}
  
  \item (Bounded data) There exist constants $\ba, \bc > 0$ such that $\| \tma
  \|_{\infty} \leq \ba$ and $| c | \leq \bc$ almost surely. \label{A2}
  
  \item (Linear resource) The average resource $\tmd = T^{-1} \tmb$ satisfies $\bld \cdot \1 \leq
  \tmd \leq \bd \cdot \1$, where $0 < \bld \leq \bd$.
  \label{A3}
\end{enumerate}

\ref{A1} to \ref{A3} are standard and minimal in the OLP literature \cite{10.48550/arxiv.2011.10124,10.48550/arxiv.2003.02513,gao2023solving}, and it is known that online subgradient method (\Cref{alg:subgrad}) with constant stepsize $ \alpha_t \equiv \Ocal{(1/\sqrt{T})}$ achieves $\mathcal{O} ( \sqrt{T}
)$ regret.

\begin{thm}[Sublinear regret benchmark \cite{gao2023solving,10.48550/arxiv.2003.02513}] \label{thm:regret-bench}
  Under \ref{A1} to \ref{A3}, online subgradient method \eqref{eqn:osgm} with
  $\alpha_t \equiv \sqrt{\frac{2 \uc}{m \ld ( \ua + \ud )^2}}
  \cdot \tfrac{1}{\sqrt{T}}$ outputs $\hat{\x}_T$ such that
  \[ \mathbb{E} [ r ( \hat{\x}_T ) + v ( \hat{\x}_T
     ) ] \leq \tfrac{m ( \ua + \ud )^2}{\ld} + \sqrt{m}
     ( \ua + \ud ) + \sqrt{\tfrac{m \uc}{2 \ld}} ( \ua + \ud
     ) \sqrt{T} =\mathcal{O} ( \sqrt{T} ) . \]
\end{thm}
\Cref{thm:regret-bench} will be used as a benchmark for our results. Under \ref{A1} to \ref{A3}, $\mathcal{O} ( \sqrt{T} )$ regret has been shown to
achieve the lower bound \cite{arlotto2019uniformly}. Under further assumptions such as non-degeneracy, LP-based OLP algorithms can efficiently leverage this structure to achieve $\mathcal{O} (\log T)$ and $\mathcal{O}(1)$ regret, respectively, in the continuous \cite{li2022online} and finite support settings \cite{10.48550/arxiv.2101.11092}. However, how first-order methods can efficiently exploit these structures remains less explored. This paper establishes a new online learning framework to resolve this issue. In particular, we consider the $\gamma$-error bound condition from the optimization literature and summarize our main results below:

\begin{thm}[\Cref{thm:final}, informal]
Suppose $f(\y)$ satisfies $\gamma$-dual error bound condition ($\gamma \geq 1$) and that $\Ycal^\star = \{\y^\star\}$ is a singleton. Then, our framework achieves 
\[\mathbb{E} [ r ( \hat{\x}_T ) + v ( \hat{\x}_T
     ) ] \leq \Ocal(T^{\frac{\gamma - 1}{2 \gamma - 1}} \log T)\]
using first-order methods.
\end{thm}

It turns out the dual error bound is key to improved regret for first-order methods. In the next section, we formally define the dual error bound condition and introduce its consequences.

%% file: sec_errorbnd.tex
\section{Dual error bound and subgradient method}
\label{sec:errbnd}

In this section, we discuss the dual error bound condition that allows first-order OLP
algorithms to go beyond $\mathcal{O} ( \sqrt{T} )$ regret. We also introduce and explain several important implications of the error bound condition for the subgradient method. Unless specified,  we restrict $\y^{\star} =
\min_{\y \in \mathcal{Y}^{\star}} \| \y \|$ and $\y^{\star}_T =
\min_{\y \in \mathcal{Y}^{\star}_T} \| \y \|$ to be the unique minimum-norm
solution to the distribution dual problem \eqref{eqn:distribution-dual} and the sample dual problem \eqref{olp:dual-sample}.

\subsection{Dual error bound condition}

Our key assumption, also known in the literature as the H\"{o}lder error bound condition \cite{johnstone2020faster}, is
stated as follows.

\begin{enumerate}[leftmargin=30pt, label=\textbf{A\arabic*.},ref=\rm{\textbf{A\arabic*}},start=4]
  \item (Dual error bound) $f ( \y ) - f ( \y^{\star} ) \geq
\mu \cdot \dist ( \y, \mathcal{Y}^{\star} )^{\gamma}$ for
all $\y \in \mathcal{Y}= \{ \y : \y \geq \mathbf{0}, \| \y \| \leq
\tfrac{\uc + \ld}{\ld} \}, \gamma \in [1, \infty)$. \label{A4}
\end{enumerate}
The assumption \ref{A4} states a growth condition in terms of the expected dual function: as $\y$ leaves the distribution dual optimal set $\Ycal^\star$, the objective will grow at least at rate $\dist ( \y, \mathcal{Y}^{\star} )^{\gamma}$. It is implied by the assumptions used in the analysis of LP-based OLP algorithms, which we summarize below.

\begin{rem}
The set $\Ycal$ is chosen such that
 $\Ycal^\star \subseteq \Ycal$ since $\y^{\star} \geq \mathbf{0}$  and
\begin{equation} \label{eqn:bounded-dual}
	\ld \| \y^{\star} \| \leq  \ld \| \y^{\star} \|_1 \leq \langle \tmd, \y^{\star}
   \rangle \leq f ( \y^{\star} ) =\mathbb{E} [
   \langle \tmd, \y^{\star} \rangle + [ c - \langle \tma,
   \y^{\star} \rangle ]_+ ] \leq f ( \mathbf{0} ) \leq
   \bar{c}.
\end{equation}
Similarly, we can show that $\y_T^\star \in \Ycal$.
\end{rem}

\begin{exple}[Continuous-support, non-degeneracy \cite{li2022online,bray2019logarithmic,ma2024optimal}]
  Suppose there exist $\lambda_1, \lambda_2, \lambda_3 > 0$ such that
  \begin{itemize}[leftmargin=15pt]
    \item $\mathbb{E} [ \tma \tma^{\top} ] \succeq \lambda_1 \I$.
    
    \item $\lambda_3 | \langle \tma, \y - \y^{\star} \rangle
    | \geq | \mathbb{P} \{ c \geq \langle \tma, \y
    \rangle | \tma \} -\mathbb{P} \{ c \geq \langle
    \tma, \y^{\star} \rangle | \tma \} | \geq \lambda_2
    | \langle \tma, \y - \y^{\star} \rangle |$ for all
    $\y \in \mathcal{Y}$.
    
    \item $y_i^{\star} = 0$ for all $d_i -\mathbb{E}_{( c, \tma )}
    [ a_i \mathbb{I} \{ c > \langle \tma, \y^{\star}
    \rangle \} ] > 0$ for all $i$.
  \end{itemize}
  Then $\mathcal{Y}^{\star} = \{ \y^{\star} \}$ and $f ( \y
  ) - f ( \y^{\star} ) \geq \frac{\lambda_1 \lambda_2}{2}
  \| \y - \y^{\star} \|^2$. Here $\diam (\mathcal{Y}^{\star})
  = 0, \mu = \frac{\lambda_1 \lambda_2}{2}$ and $\gamma = 2$.
\end{exple}

\begin{exple}[Finite-support, non-degeneracy \cite{10.48550/arxiv.2101.11092}]
  Suppose $( c, \tma )$ has finite support. Then there exists $\mu
  > 0$ such that \[f ( \y ) - f ( \y^{\star} ) \geq \mu \cdot
  \dist ( \y, \mathcal{Y}^{\star} ).\]Here $\diam
  (\mathcal{Y}^{\star}) \geq 0$, $\gamma = 1$, and $\mu$ is determined by the data distribution. If the expected LP is non-degenerate, then $\diam(\Ycal^\star) = 0$.
\end{exple}

\begin{exple}[General growth]
  Suppose $\Ycal^\star \subseteq \inte
  (\mathcal{Y})$ and there exist $\lambda_4, \lambda_5 > 0$ such that
  \begin{itemize}[leftmargin=15pt]
    \item $\mathbb{E} [ | \langle \tma, \y \rangle | ] \succeq \lambda_4 \| \y \|$ for all $\y \in \Ycal$.
    
    \item $| \mathbb{P} \{ c \geq \langle \tma, \y
    \rangle | \tma \} -\mathbb{P} \{ c \geq \langle
    \tma, \y^{\star} \rangle | \tma \} | \geq \lambda_5
    | \langle \tma, \y - \y^{\star} \rangle |^p, p \in [1, \infty)$.
  \end{itemize}
  Then $\mathcal{Y}^{\star} = \{ \y^{\star} \}$ and $f ( \y
  ) - f ( \y^{\star} ) \geq \frac{\lambda_4^{p+1} \lambda_5}{2(p+1)} \| \y -
  \y^{\star} \|^{p + 1}$. Here $\diam (\mathcal{Y}^{\star}) = 0, \mu =
  \frac{\lambda_4^{p+1} \lambda_5}{2 (p + 1)}$ and $\gamma = p + 1$.
\end{exple}

We leave the detailed verification of the results in the appendix \Cref{app:verify}. \goodnewline

While \ref{A4} is implied by the non-degeneracy assumptions in the literature, \ref{A4} does not rule out degenerate LPs. An LP can be degenerate but still satisfy the error bound. Therefore, \ref{A4} is weaker than the existing assumptions in the OLP literature. The error bound has several important consequences on our algorithm design, which we summarize next.

\subsection{Consequences of the dual error bound}

In the stochastic input setting, the online subgradient method (\Cref{alg:subgrad}) can be viewed as stochastic subgradient method (\sgm), where the error bound condition is widely studied in the optimization literature \cite{yang2018rsg,johnstone2020faster,xu2017stochastic}. We will use three implications of \ref{A4} to facilitate OLP algorithm design. The first implication is the existence of efficient first-order methods that learn $\mathcal{Y}^{\star}$.
\begin{lem}[Efficient learning algorithm]
\label{lem:alg-conv}
Under \ref{A1} to \ref{A4}, there exists a first-order method $\Acal_L$
  such that after $T$ iterations, it outputs some $\bar{\y}^{T_\varepsilon + 1} \in \{ \y : \y \geq \mathbf{0}, \| \y \| \leq
\tfrac{\uc}{\ld} \}$  such that
  for all $T_\varepsilon \geq \mathcal{O} ( \varepsilon^{- 2 (1 - \gamma^{- 1})} \log
  ( \tfrac{1}{\varepsilon} ) \log ( \tfrac{1}{\delta} )
  )$,
  \[ f ( \bar{\y}^{T_\varepsilon + 1} ) - f ( \y^{\star} ) \leq
     \varepsilon \]
  with probability at least $1 - \delta$. Moreover, for all $T_\varepsilon \geq \mathcal{O} (
  \tfrac{1}{\mu} \varepsilon^{- 2 (1 - \gamma^{- 1})} \log (
  \tfrac{1}{\varepsilon} ) \log ( \tfrac{1}{\delta} )
  )$,
  \[ \dist ( \bar{\y}^{T_\varepsilon + 1}, \mathcal{Y}^{\star} )^{\gamma}
     \leq \varepsilon \]
  with probability at least $1 - \delta$.
  \end{lem}

\Cref{lem:alg-conv} shows that there is an efficient learning algorithm (in particular, \Cref{alg:assg} in the appendix) that learns an approximate dual optimal solution $\hat{\y}$ with suboptimality $\varepsilon$ at sample complexity $\mathcal{O} ( \varepsilon^{- 2 (1 - \gamma^{- 1})} \log
  ( \tfrac{1}{\varepsilon} ))$. The sample complexity increases as the growth parameter $\gamma$ becomes larger. Moreover, \ref{A4} allows us to transform the dual suboptimality into the distance to optimality: $\dist ( \hat{\y}, \mathcal{Y}^{\star} )^{\gamma}
     \leq \varepsilon$. Back to the context of OLP, when the growth parameter $\gamma$ is small, it is possible to learn the distribution optimal solution with a small amount of customer data. For example, with $\gamma = 1$ and $\varepsilon = \delta = 1/T$, we only need the information of $\Ocal(\log^2 ( T ))$ customers to learn a highly accurate approximate dual solution satisfying $\dist(\hat{\y}, \Ycal^\star) \leq T^{-1}$. In other words, $\gamma$ characterizes the complexity or difficulty of the distribution of $(c, \tma)$; smaller $\gamma$ implies that the distribution is easier to learn.\goodnewline
     
The second implication of \ref{A4} comes from the stochastic optimization literature \cite{liu2023revisiting}: suppose the subgradient method (\Cref{alg:subgrad}) runs with constant stepsize $\alpha$, then the last iterate will end up in a noise ball around the optimal set, whose radius is determined by the initial distance to optimality $\dist ( \y^1,
\mathcal{Y}^{\star} )$ and the subgradient stepsize $\alpha$.

\begin{lem}[Noise ball and last iterate convergence]
\label{lem:vr}
  Under \ref{A1} to \ref{A3}, suppose \Cref{alg:subgrad} uses $\alpha_t \equiv \alpha$ for
  all $t$, then
  \[ \mathbb{E} [ f ( \y^{T + 1} ) - f ( \y^{\star} )
     ] \leq \mathcal{O}(\tfrac{\Delta^2}{\alpha T} + \alpha \log T), \]
where $\Delta := \dist ( \y^1, \mathcal{Y}^{\star} )$.
Moreover, if \ref{A4} holds, then
  \[ \mathbb{E} [ \dist ( \y^{T + 1}, \mathcal{Y}^{\star} )^\gamma
     ] \leq \mathcal{O} \big( \tfrac{\Delta^2}{\mu \alpha T} + \tfrac{\alpha}{\mu} \log T \big) . \]
\end{lem}

To demonstrate the role of \Cref{lem:vr} in our analysis. Suppose  $\Delta$ is sufficiently small and $\alpha = \Ocal(\Delta)$ is fixed. Then applying \Cref{lem:vr} with $T = 1, \ldots $ shows that all the iterates generated by \Cref{alg:subgrad} will satisfy 
\[\mathbb{E} [ \dist ( \y^{t + 1}, \mathcal{Y}^{\star} )^\gamma
     ] \leq \mathcal{O} (\Delta \log T), \quad t = 1,\ldots, T. \]
In other words, if $\y^{1}$ is close to the optimal set $\mathcal{Y}^{\star}$, then a proper choice of subgradient stepsize will keep all the iterates in a noise ball around $\mathcal{Y}^{\star}$. This noise ball is key to our improved reget guarantee.\goodnewline

The last implication, which connects the behavior of the subgradient method and OLP, states
that the hindsight optimal dual solution $\y^{\star}_T$ will be close to
$\mathcal{Y}^{\star}$.

\begin{lem}[Dual convergence] \label{lem:dual-conv}
  Under \ref{A1} to \ref{A4}, for any $\y^{\star}_T \in \Ycal^\star_T$, we have
  \[ \mathbb{E} [ \dist ( \y^{\star}_T, \mathcal{Y}^{\star}
     )^\gamma ] \leq \Ocal(\sqrt{\tfrac{{\log T}}{\mu^2 {T}}}). \]
\end{lem}

\Cref{lem:dual-conv} states a standard dual convergence result when \ref{A4} is present. This type of result is key to the analysis of the LP-based methods \cite{li2022online,bray2019logarithmic,ma2024optimal}. Although our analysis will not explicitly invoke \Cref{lem:dual-conv}, it provides sufficient intuition for our algorithm design: suppose $\diam (\Ycal^\star) = 0$, then the hindsight $\y^{\star}_T$, which has no regret, will be in an $o(1)$ neighborhood around $\y^{\star}$. In other words, if we have prior knowledge of the customer distribution (thereby, $\y^{\star}$), we can localize around $\y^{\star}$ since we know $\y^{\star}_T$ will not be far off. Moreover, according to \Cref{lem:alg-conv} and \Cref{lem:vr}, the subgradient method has the ability to get close to, and more importantly, to stay in proximity (the noise  ball) around
$\mathcal{Y}^{\star}$. Intuitively, if $\y^1$ is in an $o(1)$ neighborhood of $\Ycal^\star$, then we can adjust the stepsize of the subgradient method so that the online decision-making happens in an $o (1)$ neighborhood around $\mathcal{Y}^{\star}$, and better performance is naturally expected. Even if $\diam(\mathcal{Y}^{\star}) > 0$, the same argument still applies and can improve performance by a constant. In the next section, we formalize the aforementioned intuitions and establish a general framework for first-order methods to achieve better performance.

%\begin{figure}[h]
%\centering
%	\includegraphics[scale=0.4]{figs/dual_conv.pdf}
%	\caption{Illustration of the idea. When $\y^1$ is close to $\y^\star$, then the subgradient method with smaller stepsize stays in the noise ball around $\y^\star$. Since $\y^\star_T$ is also in an $o(1)$ neighborhood of $\y^\star$, we get improved regret.
%	\label{fig:dual-conv}} 
%\end{figure}

%% file: sec_algorithm.tex
\section{Improved regret with first-order methods}
\label{sec:algo}

This section formalizes the intuitions established in \Cref{sec:errbnd} and introduces a general framework that allows first-order methods to go beyond $\Ocal(\sqrt{T})$ regret.

\subsection{Regret decomposition and localization} \label{sec:regret-decomp}
We start by formalizing the intuition that if $\y^1$ is sufficiently close to $\y^\star$, then adjusting the stepsize of the subgradient method allows us to make decisions in a noise ball around $\y^\star$ and achieve improved performance.

\begin{lem}[Regret]\label{lem:regret}
  Under \ref{A1} to \ref{A4}, if $\|\y^1\| \leq \frac{\uc}{\ld}$, then the output of \Cref{alg:subgrad} satisfies
  \[ \mathbb{E} [ r ( \hat{\x}_T) ] \leq \tfrac{m
     ( \ba + \bd )^2 \alpha}{2} T + \tfrac{\mathsf{R}}{\alpha}
     [ \| \y^1 - \y^{\star} \| +\mathbb{E} [ \|
     \y^{T + 1} - \y^{\star} \| ] ], \]
  where $\mathsf{R} = \frac{\bc}{\bld} + \big[ \frac{m ( \ba + \bd
  )^2}{2 \bld} + \sqrt{m} ( \ba + \bd ) \big] \alpha$.
\end{lem}

\begin{lem}[Violation] \label{lem:violation}
Under the same conditions as \Cref{lem:regret}, the output of  \Cref{alg:subgrad} satisfies
  \[ \mathbb{E} [ v ( \hat{\x}_T) ] \leq
     \tfrac{1}{\alpha} [ \| \y^1 - \y^{\star} \| +\mathbb{E}
     [ \| \y^{T + 1} - \y^{\star} \| ] ]. \]
\end{lem}

\begin{rem}
	Note that in our analysis, $\alpha$ will always be $o(1)$ if $T$ is sufficiently large. Therefore, we can consider $\mathsf{R}$ as a constant without loss of generality.
\end{rem}

Putting \Cref{lem:regret} and \Cref{lem:violation} together, the performance of
\Cref{alg:subgrad} is characterized by
\begin{equation} \label{eqn:performance}
	\mathbb{E} [ r ( \hat{\x}_T) + v ( \hat{\x}_T
   ) ] \leq \mathcal{O} ( \alpha T + \tfrac{1}{\alpha}
   \| \y^1 - \y^{\star} \| + \tfrac{1}{\alpha} \mathbb{E} [
   \| \y^{T + 1} - \y^{\star} \| ] ).
\end{equation}

In the standard OLP analysis, it is only possible to ensure boundedness of
$\| \y^1 - \y^{\star} \|$ and $\| \y^{T + 1} - \y^{\star}
\|$. In other words,
\[ \mathbb{E} [ r ( \hat{\x}_T) + v ( \hat{\x}_T
   ) ] \leq \mathcal{O} ( \alpha T + \tfrac{1}{\alpha}
   ) \]
and the optimal trade-off at $\alpha =\mathcal{O} ( 1 / \sqrt{T} )$
gives $\mathcal{O} ( \sqrt{T} )$ performance in \Cref{thm:regret-bench}. However, under \ref{A4}, our analysis more accurately characterizes the behavior of the subgradient method, and we can do much better when $\y^1$ is close to
$\y^{\star}$: suppose for now that $\diam (\mathcal{Y}^{\star}) = 0$
($\mathcal{Y}^{\star}$ is a singleton) and that $\| \y^1 - \y^{\star} \| = 0$.
\Cref{lem:vr} with $\Delta = 0$ ensures that
\begin{equation}\label{eqn:trade-off}
	\mathbb{E} [ \| \y^{T + 1} - \y^{\star} \| ]
   =\mathcal{O} ((\alpha \log T)^{1 / \gamma}).
\end{equation}
Plugging \eqref{eqn:trade-off} back into \eqref{eqn:performance},
\begin{equation}\label{eqn:naive-trade-off}
	\mathbb{E} [ r ( \hat{\x}_T) + v ( \hat{\x}_T
   ) ] \leq \mathcal{O} ( \alpha T + \tfrac{1}{\alpha}
   \alpha^{1 / \gamma} (\log T)^{1 / \gamma} ) =\mathcal{O} (
   \alpha T + \tfrac{1}{\alpha^{1 - 1 / \gamma}} (\log T)^{1 / \gamma} )
\end{equation}
and taking $\alpha = T^{-\frac{\gamma}{2 \gamma - 1}}$ gives
\[ \mathbb{E} [ r ( \hat{\x}_T) + v ( \hat{\x}_T
   ) ] \leq \mathcal{O} ( T^{\frac{\gamma - 1}{2 \gamma - 1}}
   (\log T)^{1 / \gamma} ) . \]

This simple argument provides two important observations.

\begin{itemize}[leftmargin=10pt]
\item When $\gamma < \infty$ and $\diam(\Ycal^\star) = 0$, the knowledge of $\y^{\star}$ significantly improves the performance of first-order methods by shrinking the stepsize of the subgradient method from $\Ocal(1/\sqrt{T})$ to $\Ocal(T^{-\frac{\gamma}{2\gamma-1}})$: small stepsize implies localization around $\y^\star$. If $\gamma = 1$, we achieve $\mathcal{O} (\log T)$ regret; if $\gamma \rightarrow \infty$, we
recover $\mathcal{O} ( \sqrt{T} )$ regret.
\item Even if $\y^{\star}$ is known, the optimal strategy is not taking $\alpha = 0$ and staying at $\y^{\star}$. Instead, $\alpha$ should be chosen according to $\gamma$, the strength of the error bound.
\end{itemize}

In summary, when $\y^1$ is close to $\y^{\star}$, we achieve improved performance guarantees through localization. The smaller $\gamma$ is, the smaller stepsize we take, and finally, the better regret we achieve. This observation matches \Cref{lem:alg-conv}: when a distribution is ``easy'', we can trust $\y^{\star}$ and stay close to it.\goodnewline

Although it is sometimes reasonable to assume prior knowledge of $\y^\star$ beforehand, it is not always a practical assumption. Therefore, a natural strategy is learning it online from the customers. It is where the efficient learning algorithm from \Cref{lem:alg-conv} comes into play and leads to an exploration-exploitation framework.

\subsection{Exploration and exploitation}

When $\y^{\star}$ is not known beforehand, \Cref{lem:alg-conv} shows first-order methods can learn it from data and an exploration-exploitation strategy (\Cref{alg:two-phase}) is easily
applicable: specify a target accuracy $\Delta$ and define
\begin{align}
 \text{Exploration horizon  ~} T_e \assign{} & \mathcal{O} ( \tfrac{1}{\Delta^{2 (\gamma - 1)}} \log
  ( \tfrac{1}{\Delta^{\gamma}} ) \log \tfrac{1}{T^{- 2 \gamma}}
  ) = \mathcal{O} ( \tfrac{1}{\Delta^{2 (\gamma - 1)}} \log (
  \tfrac{1}{\Delta^{\gamma}} ) \log T)
  \label{eqn:len-explore}\\
 \text{Exploitaition horizon  ~}  T_p \assign{} & T - T_e, \nonumber
\end{align}

where $T_e$ is obtained by taking $\varepsilon = \Delta^{\gamma}$ and $\delta
= T^{- 2 \gamma}$ in \Cref{lem:alg-conv}. Without loss of generality, we assume that $T_e$ is
an integer and that $T \gg T_e$. Then \Cref{lem:alg-conv} guarantees $\dist ( \bar{\y}^{T_e + 1}, \mathcal{Y}^{\star} )
   \leq \Delta$ with probability at least $1 - T^{- 2 \gamma}$. In the exploitation phase, we use the subgradient method (\Cref{alg:subgrad}) with a properly
configured stepsize to localize around $\Ycal^\star$ and achieve better performance. \Cref{lem:two-phase} characterizes the behavior of this two-phase algorithm \Cref{alg:two-phase}.

\begin{algorithm}[h]
\caption{Exploration-exploitation  \label{alg:two-phase}}	
\KwIn{$\y^1 = \mathbf{0}$ (no prior knowledge), learning algorithm $\Acal_L$ in \Cref{lem:alg-conv}, exploration length $T_e$}

\textbf{explore} $\bar{\y}^{T_e + 1} \approx \y^\star$  for $t = 1$ to $T_e$ with $\Acal_L$

\textbf{exploit} \For{$t$ = \rm{$T_e + 1$ to $T$ }}{
Run \Cref{alg:subgrad} starting with $\y^{T_e + 1} = \bar{\y}^{T_e + 1}$ with proper stepsize.
}
\end{algorithm}

\begin{lem} \label{lem:two-phase}
  Under the same assumptions as \Cref{lem:regret}, the output of
  \Cref{alg:two-phase} satisfies
  \[ \mathbb{E} [ r ( \hat{\x}_T ) + v (
     \hat{\x}_T ) ] \leq V (T_e)+  \mathcal{O} ( \alpha T_p +
     \tfrac{\Delta}{\alpha} + \tfrac{\Delta^{2 / \gamma}}{\alpha^{1 / \gamma +
     1} T_p^{1 / \gamma}} + \alpha^{1 / \gamma - 1}  (\log T)^{1 / \gamma} +
     \tfrac{{\diam} (\mathcal{Y}^{\star})}{\alpha} + \tfrac{1}{\alpha T^{2
     \gamma}} + \tfrac{1}{\alpha^{1 / \gamma + 1} T_p^{1 / \gamma} T^2}
     ), \]
where $V (T_e) \assign \mathbb{E} [ \| [ \textstyle \sum_{t = 1}^{T_e} (
   \tma_t x^t - \tmd ) ]_+ \| + \textstyle \sum_{t = 1}^{T_e} f (
   \y^{\star} ) - c_t x^t ]$ is the performance metric in the exploration phase.
\end{lem}

\Cref{lem:two-phase} presents two trade-offs:
\begin{itemize}[leftmargin=10pt]
  \item \textit{Trade-off between exploration and exploitation}.
  
  A high accuracy approximate dual solution $\dist(\bar{\y}^{T_e + 1}, \Ycal^\star) = \Delta \approx 0$ allows localization and improves the performance in exploitation. However,
  reducing $\Delta$ requires a longer exploration phase and larger $V_e (T_e)$.
  
  \item \textit{Trade-off of stepsize within the exploitation phase}.

As in \eqref{eqn:naive-trade-off}, the following terms dominate the performance in the exploitation phase
  \[ \alpha T_p + \tfrac{\Delta}{\alpha} + \tfrac{\Delta^{2 /
     \gamma}}{\alpha^{1 / \gamma + 1} T_p^{1 / \gamma}} + \alpha^{1 / \gamma -
     1}  (\log T)^{1 / \gamma} + \tfrac{{\diam}
     (\mathcal{Y}^{\star})}{\alpha} \]
  and we need to set the optimal $\alpha$ based on $(T_p, \Delta, \gamma,
  {\diam} (\mathcal{Y}^{\star}))$.
\end{itemize}
Note that we haven't specified the expression of $V (T_e)$, since it depends on the dual sequence used for decision-making in the exploration phase. Ideally,
$V (T_e)$ should grow slowly in $T_e$ so that exploration provides a high-quality solution without compromising the overall algorithm performance. One natural idea is to make decisions
based on the dual solutions produced by the efficient learning algorithm in \Cref{lem:alg-conv}. This is exactly what LP-based methods do \cite{li2022online}. However, as we
will demonstrate in the next section, a good first-order learning algorithm can be inferior for decision-making. This counter-intuitive observation motivates the idea of decoupling learning and decision-making, and finally provides a general framework for first-order methods to go beyond $\mathcal{O}
( \sqrt{T} )$ regret.

\subsection{Dilemma between learning and decision-making}
\Cref{lem:two-phase} requires controlling $V (T_e)$, the performance metric during exploration, by specifying $\{ \y^t \}_{t = 1}^{T_e}$ used for decision-making. It seems natural to adopt $\{ \y^t_L \}_{t = 1}^{T_e}$, the dual iterates produced
by $\mathcal{A}_{L}$ for decision-making, and one may also wonder whether
running $\mathcal{A}_{L}$ for decision-making over the whole horizon $T$ leads to further improved performance guarantees. However, this is not the case: using a good learning algorithm for decision-making leads to
worse performance guarantees. To demonstrate this issue, we give a concrete example and consider the following one-dimensional multi-secretary online LP:
\begin{equation}\label{eqn:olsecretary}
	\max_{0 \leq x^t \leq 1}  \textstyle   ~~\sum_{t = 1}^T c_t x^t ~~ \text{subject to}
   ~~ \sum_{t = 1}^T x^t \leq \tfrac{T}{2},
\end{equation}
where $\{ c_t \}_{t = 1}^T$ are sampled uniformly from $[0, 1]$. For this problem, $\mu = \tfrac{1}{2}, \gamma = 2$, and $y^{\star} =
\tfrac{1}{2}$ is unique. Subgradient method with stepsize
$\alpha_t = 1 / (\mu t)$ satisfies the convergence result of \Cref{lem:alg-conv} is a suitable candidate for
$\mathcal{A}_{L}$:

\begin{lem} \label{lem:subgrad-conv}
For the multi-secretary problem \eqref{eqn:olsecretary}, subgradient method
  \[ y^{t + 1} = [y^t - \alpha_t g^t]_+ \]
  with stepsize $\alpha_t = \tfrac{1}{\mu t}$ satisfies  $|y^{T + 1} - y^\star|^2 \leq \Ocal{(\frac{\log \log T + \log (1/\delta)}{T} )}$ at least with probability $1- \delta$.
\end{lem}

\Cref{lem:subgrad-conv} suggests that using $\mathcal{A}_{L}$, we indeed
approximate $y^{\star}$ efficiently. However, to approximate $y^{\star}$ to high accuracy, the algorithm will inevitably take small stepsize $\alpha_t$ when $t \geq \Omega(T)$. Following our discussion in \Cref{sec:regret-decomp}, even with perfect information of $y^{\star}$,  the online algorithm for $\gamma = 2$ should remain adaptive to the environment by taking stepsize $\Ocal(T^{-2/3})$. Taking $\Ocal(1/T)$ stepsize nullifies this adaptivity, and the most direct consequence of lack of adaptivity is that, when the learning
algorithm deviates from $y^{\star}$ due to noise, the overly small stepsize
will take the algorithm a long time to get back. From an optimization
perspective, this does not necessarily affect the quality of the final output
$y^{T + 1}$, since we only care about the quality of the final output. However, as a decision-making algorithm, the regret will
\textit{accumulate} when the algorithm tries to get back. This observation shows a
clear distinction between stochastic optimization and online decision-making. \Cref{prop:slow_recover_sgd} formalizes the aforementioned consequence:

\begin{lem}
    \label{prop:slow_recover_sgd}
    Denote $y^t$ as the estimated dual solution for the online secretary problem \eqref{eqn:olsecretary} at time $t$ by the subgradient method with stepsize $1 / (\mu t)$ specified in \Cref{lem:subgrad-conv}. If there exists $t_0\geq T/10+1$ such that $y^{t_0}\geq y^\star+\frac{1}{\sqrt{T}}$, then $\mathbb{E}[y^t|y^{t_0}]
        \geq
        y^\star + \frac{1}{20\sqrt{T}}$
     for all $t\geq t_0$.
\end{lem}
As a consequence, a good learning algorithm, due to its lack of adaptivity,
is a bad decision-making algorithm:

\begin{prop}[Dilemma between learning and decision-making]
    \label{prop:sgd_badregret}
If subgradient method with stepsize $1 / (\mu t)$ is used for decision-making, it cannot achieve $\mathcal{O}(T^{\beta})$ regret and constraint violation simultaneously for any $\beta<\frac{1}{2}$.
\end{prop}

Although our example only covers $\gamma = 2$, similar issues happen for other values of $\gamma$: the stepsize used by learning algorithms (\Cref{lem:alg-conv}) near convergence are much smaller than the optimal choice dictated by \eqref{eqn:naive-trade-off}. This argument reveals a dilemma between learning and decision-making: a learning algorithm needs a small stepsize to achieve high accuracy, while a decision-making algorithm needs a larger stepsize to maintain adaptivity to the environment. This dilemma is inevitable for a single first-order method. However, the low computation cost of first-order methods opens up another way: it is feasible to use two separate algorithms for learning and decision-making.

\subsection{Decoupling learning and decision-making}

As discussed, a single first-order method may not simultaneously achieve good regret and
accurate approximation of $\Ycal^{\star}$. However, this dilemma
can be easily addressed if we decouple learning and decision-making and employ two first-order methods for learning and
decision-making, respectively. The iteration cost of first-order methods is
inexpensive, so it is feasible to maintain multiple of them to take the best
of both worlds: the best possible learning algorithm $\mathcal{A}_L$ and
decision algorithm $\mathcal{A}_D$. Back to the exploration-exploitation framework, in the exploration phase, we can
take $\mathcal{A}_D$ to be the same
subgradient method with constant stepsize, which we know at least guarantees $\Ocal(\sqrt{T_e})$ performance for horizon length $T_e$. The algorithm maintains two paths
of dual sequences in the exploration phase, and when exploration is over, the
solution learned by $\mathcal{A}_L$ is handed over to the exploitation phase and
$\mathcal{A}_D$ adjusts stepsize based on the trade-off in \Cref{lem:two-phase}. Since the subgradient methods with different stepsizes are used for decision-making in both exploration and exploitation, the actual effect of the framework is to \textit{restart} the subgradient method. The final algorithm is presented in \Cref{alg:final} (\Cref{fig:twopath}).
\begin{figure*}[h]
\centering
\includegraphics[scale=0.6]{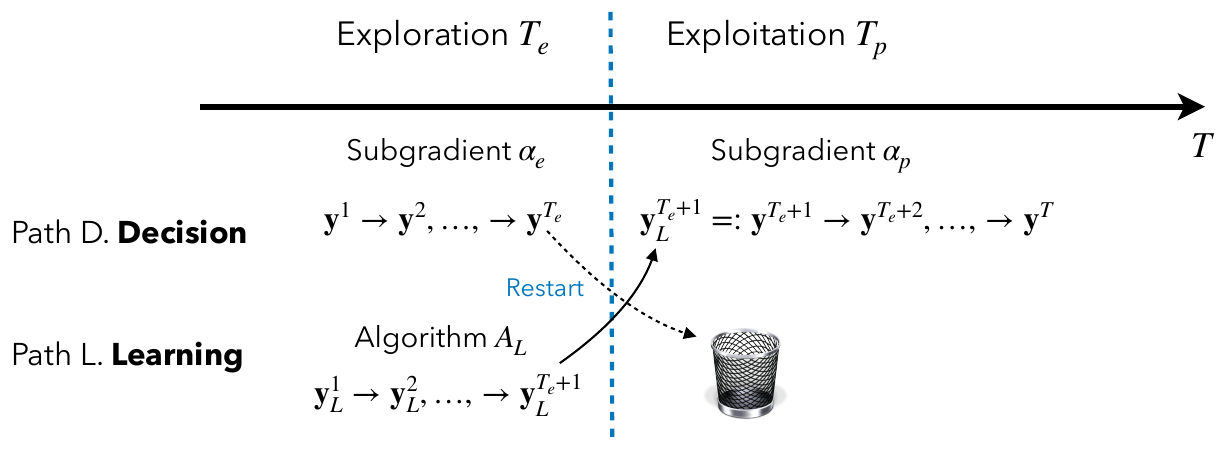}
\caption{Exploration phase sends $\y^{T_e + 1}$ into a neighborhood of $\Ycal^\star$, and in the exploitation phase, $\{\y^t\}_{T_e+1}^T$ localizes in this neighborhood with adaptivity to make adjustments.\label{fig:twopath}}
\end{figure*}
\begin{algorithm}[h]
\caption{Exploration-Exploitation, decoupling learning and decision-making, and localization}
\label{alg:final}
\KwIn{$\y^1 = \mathbf{0}$ (no prior knowledge), learning algorithm $\Acal_L$ in \Cref{lem:alg-conv}, \\
\quad\quad\quad\quad decision algorithm $\Acal_D = $ \Cref{alg:subgrad}, exploration length $T_e$}
\textbf{explore} \For{$t$ = \rm{$1$ to $T_e$ }}{
Run \Cref{alg:subgrad} with stepsize $\alpha_e$.

Run $\Acal_L$ and learn $\bar{\y}^{T_e + 1} \approx \y^\star$.
}

\textbf{exploit} \For{$t$ = \rm{$T_e + 1$ to $T$ }}{
Run \Cref{alg:subgrad} starting with $\y^{T_e + 1} = \bar{\y}^{T_e + 1}$ with stepsize $\alpha_p$.
}
\end{algorithm}

Decoupling learning and decision-making, \Cref{lem:final} characterizes the performance of the whole framework.
\begin{lem} \label{lem:final}
  Under the same assumptions as \Cref{lem:two-phase}, the output of \Cref{alg:final} satisfies
  \[ V (T_e) \leq \mathcal{O} ( \tfrac{1}{\alpha_e} + \alpha_e T_e
     ) \]
and we have the following performance guarantee:
\begin{align}
	\mathbb{E} [ r ( \hat{\x}_T ) + v (
     \hat{\x}_T ) ] \leq{} & \mathcal{O} \Big(\tfrac{1}{\alpha_e} + \alpha_e T_e + \alpha_p T_p +
     \tfrac{\Delta}{\alpha_p} + \tfrac{\Delta^{2 / \gamma}}{\alpha_p^{1 / \gamma +
     1} T_p^{1 / \gamma}} \nonumber \\
    & ~~~~\quad + \alpha_p^{1 / \gamma - 1}  (\log T)^{1 / \gamma} +
     \tfrac{{\diam} (\mathcal{Y}^{\star})}{\alpha_p} + \tfrac{1}{\alpha_p T^{2
     \gamma}} + \tfrac{1}{\alpha_p^{1 / \gamma + 1} T_p^{1 / \gamma} T^2}
     \Big).
\end{align}
\end{lem}

After balancing the trade-off by considering all the terms, we arrive at \Cref{thm:final}. 
\begin{thm}[Main theorem] \label{thm:final}
  Under the same assumptions as \Cref{lem:final} and suppose $T$ is sufficiently large.
If $\diam(\mathcal{Y}^{\star}) = 0$, then with
  \[ \quad T_e = \Ocal (T^{\frac{2\gamma - 2}{2 \gamma
     - 1}} \log^2 T ), \quad \alpha_e = \Ocal \big(\tfrac{T^{-\frac{\gamma - 1}{2 \gamma - 1}}}{\log T}\big), \quad \alpha_p =
     \Ocal (T^{- \frac{\gamma}{2 \gamma - 1}}), \]
  we have
  \[ \mathbb{E} [ r ( \hat{\x}_T ) + v (
     \hat{\x}_T ) ] \leq \mathcal{O} ( T^{\frac{\gamma -
     1}{2 \gamma - 1}} \log T ) . \]
  In particular, if $\gamma = 2$, there is no $\log T$ term. 
    If $\diam (\mathcal{Y}^{\star}) > 0$, then with
  \[ T_e = \tfrac{2 \diam (\mathcal{Y}^{\star})}{2 \diam
     (\mathcal{Y}^{\star}) + 1} T, \quad \alpha_e = \sqrt{\tfrac{2 \bar{c}}{m
     (\bar{a} + \bar{d})^2 \ld} \tfrac{2 \diam (\mathcal{Y}^{\star}) +
     1}{2 \diam (\mathcal{Y}^{\star}) T}}, \quad \alpha_p =
     \sqrt{\tfrac{2 \bar{c}}{m (\bar{a} + \bar{d})^2 \ld}  \tfrac{2
     \diam (\mathcal{Y}^{\star}) (2 \diam (\mathcal{Y}^{\star}) +
     1)}{T}}, \]
  we have
  \[ \mathbb{E} [ r ( \hat{\x}_T ) + v (
     \hat{\x}_T ) ] \leq 4 \sqrt{\tfrac{m \bar{c}}{2 \ld}}
     \sqrt{2 \diam (\mathcal{Y}^{\star})} \sqrt{T} . \]
\end{thm}

\Cref{thm:final} shows that when the dual optimal set $\mathcal{Y}^{\star}$ is a singleton, first-order methods can achieve $o(\sqrt{T})$ regret using our framework. If $\diam(\mathcal{Y}^{\star})> 0$, it is still possible to achieve better regret in terms of constant when $\diam(\Ycal^\star) \ll \diam(\Ycal)$. As a realization of our framework, we recover $o(\sqrt{T})$ performance guarantees in the traditional setting of LP-based methods.

\begin{coro}
  In the non-degenerate continuous support case, we get $\mathcal{O} (T^{1 /
  3})$ performance.
\end{coro}

\begin{coro}
  In the non-degenerate finite-support case, we get $\mathcal{O} (\log T)$ performance.
\end{coro}

Again,  the intuitions behind the algorithm are simple: error bound ensures $\y_T^\star$ is close to $\Ycal^\star$ and allows online algorithm to localize in an $o(1)$ neighborhood around $\Ycal^\star$; exploration-exploitation allows us to learn from data and get close to $\Ycal^\star$; decoupling learning and decision-making, we get the best of both worlds and control the regret in the exploration phase. These pieces together make first-order methods go beyond $\mathcal{O} (\sqrt{T})$ regret.

%% file: sec_exp.tex
\section{Numerical experiments} \label{sec:exp}

This section conducts experiments to illustrate the empirical performance of our framework. We consider both the continuous and finite support settings. To benchmark our algorithms, we compare
\begin{enumerate}[leftmargin=30pt,label=\textbf{M\arabic*.},ref=\rm{\textbf{M\arabic*}},start=1]
\item Benchmark subgradient method \Cref{alg:subgrad} with constant stepsize $\Ocal{(1/\sqrt{T})}$. \label{M1}
\item Our framework \Cref{alg:final}. \label{M2}
\end{enumerate}

\begin{enumerate}[leftmargin=30pt,label=\textbf{MLP.},ref=\rm{\textbf{MLP}}]
\item State-of-the-art LP-based methods. \label{M3}
In the continuous support setting, \ref{M3} is the action-history-dependent algorithm \cite[Algorithm 3]{li2022online}; in the finite support setting, \ref{M3} is the adaptive allocation algorithm \cite[Algorithm 1]{10.48550/arxiv.2101.11092}.
\end{enumerate}

In the following, we provide the details of \ref{M2} for each setting.

\begin{itemize}[leftmargin=10pt]
    \item For the continuous support setting, $\mathcal{A}_L$ is the subgradient method with $\mathcal{O}(1/(\mu t))$ stepsize (\Cref{lem:subgrad-conv}). As suggested by \Cref{thm:final}, $\mathcal{A}_D$ is the subgradient method with stepsize $\alpha_e = 1/\sqrt{T_e} = T^{-1/3}$ in the exploration phase $T_e = T^{2/3}$. In the exploitation phase, $\mathcal{A}_D$ takes stepsize $\alpha_p = T^{-2/3}$. We always set $\mu = 1$ and do not tune it through the experiments.
    \item For the finite support setting, $\mathcal{A}_L$ is ASSG \cite{xu2017stochastic} (\Cref{alg:assg} in the appendix); $\mathcal{A}_D$ is the subgradient method with stepsize $\alpha_e = 1/\sqrt{T}$ in the exploration phase $T_e = 50 \log T$. In the exploitation phase, $\mathcal{A}_D$ takes stepsize $\alpha_p = T^{-1}$.
\end{itemize}

\subsection{Continuous support}
\label{subsec:exp-continuous}
We generate $\{(c_t, \mathbf{a}_t)\}_{t=1}^T$ from different continuous distributions. The performance of three algorithms is evaluated in terms of $r(\hat{\x}_T) + v(\hat{\x}_T)$ (which we will call regret for simplicity). We choose $m \in \{1, 5\}$ and 10 different $T$ evenly spaced over $[10^2, 10^5]$ on $\log$-scale. 
All the results are averaged over $100$ independent random trials.
For all the distributions, each $d_i$ is sampled i.i.d. from uniform distribution $\mathcal{U}[1/3, 2/3]$. The data $\{(c_t, \mathbf{a}_t)\}_{t=1}^T$ is generated as follows: \textbf{1)}. The first distribution \cite{10.48550/arxiv.2003.02513} takes $m = 1$ and samples each $a_{it}, c_t$ i.i.d. from $\mathcal{U}[0, 2]$; \textbf{2)}. The second distribution \cite{li2022online}, takes $m=1, a_{it} = 1$, and samples each $c_t$ i.i.d. from  $\mathcal{U}[0, 1]$; \textbf{3)}. 
The third distribution takes $m = 5$ and samples $a_{it}$ from $\text{Beta}(\alpha, \beta)$ with $(\alpha, \beta) = (1, 8)$ and each $c_t$ i.i.d. from $\mathcal{U}[0, 3]$. \textbf{4)}.
The last distribution takes $m = 5$ and samples $a_{it}$ and $c_t$ i.i.d. from $\mathcal{U}[1, 6]$ and  $\mathcal{U}[0, 3]$, respectively.\goodnewline

For each distribution and algorithm, we plot the growth behavior of regret with respect to $T$. The performance statistics are normalized by the performance at $T= 10^2$. \Cref{fig:exp-con}  suggests that \ref{M2} has a better order of regret compared to \ref{M1}, which is consistent with our theory. Although \ref{M3} achieves the best performance in $r+v$, it requires significantly more computation time than \ref{M2}, since it solves an LP for each $t$. To demonstrate this empirically, we also compare the computation time of \ref{M1}, \ref{M2}, and \ref{M3}. We generate instances according to the first distribution with $m=2$ and $T \in \{10^3, 10^4, 10^5\}$. For each $(m,T)$ pair, we average the $r+v$ and computation time over $10$ independent trials and summarize the result in \Cref{tab:exp-time-con}: \ref{M3} takes more than one hour when $T=10^5$, whereas  \ref{M2} only needs $0.064$ seconds and achieves significant better regret compared to \ref{M3}. Our proposed framework effectively balances efficiency and regret performance.

\begin{figure*}[h]
\centering
\includegraphics[scale=0.25]{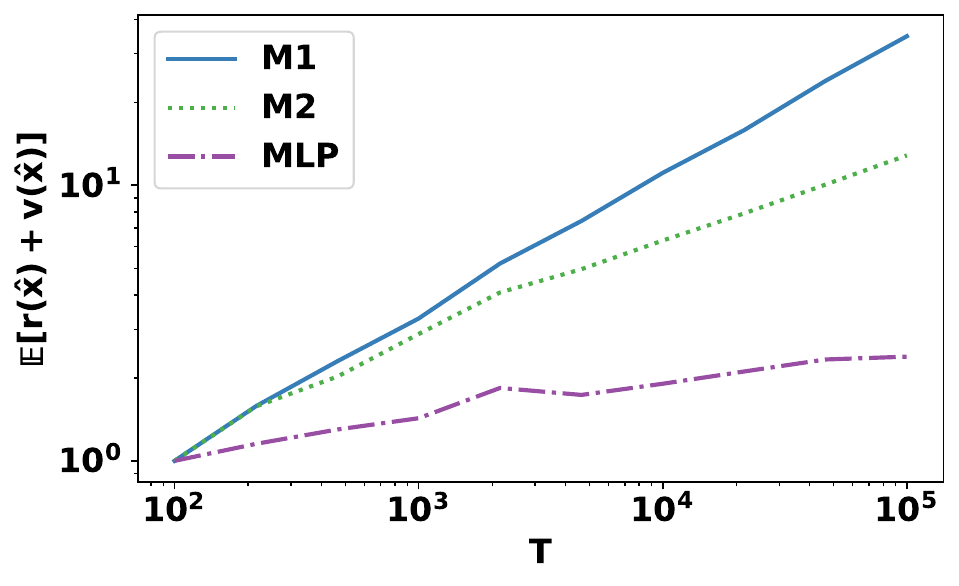}
\includegraphics[scale=0.25]{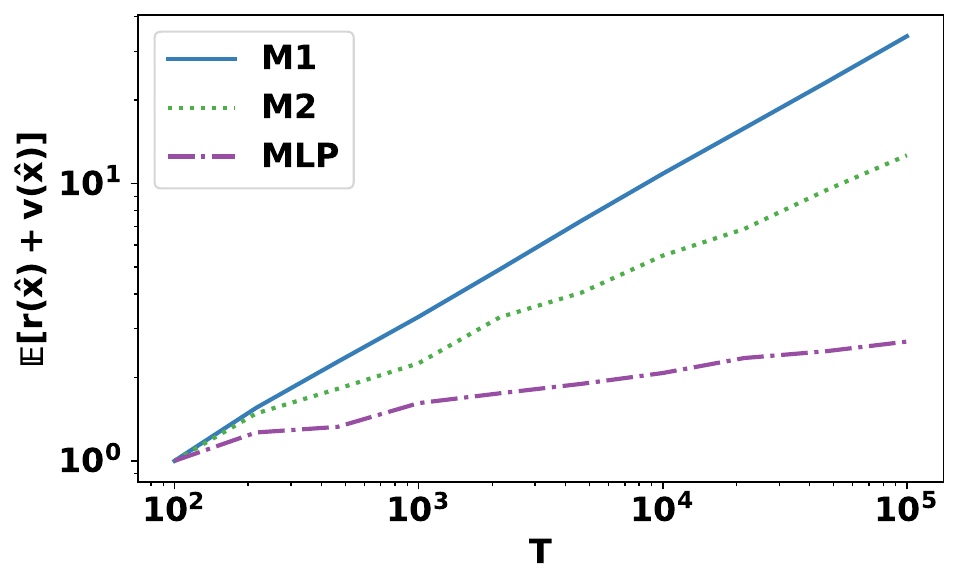}
\includegraphics[scale=0.25]{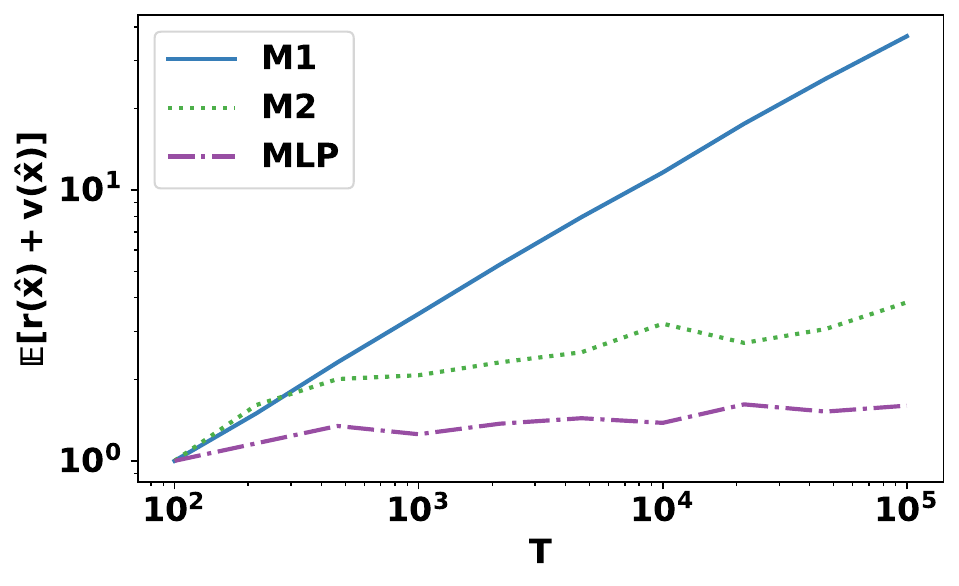}
\includegraphics[scale=0.25]{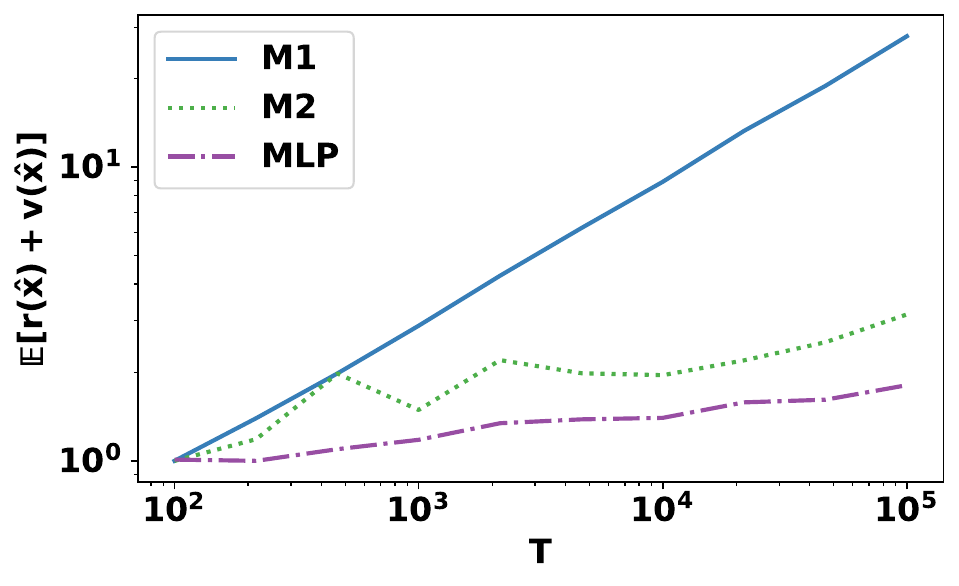}
\caption{\label{fig:exp-con} Growth of normalized 
$r(\hat{\x}_T)+v(\hat{\x}_T)$ of different algorithms under the continuous distributions.}
\end{figure*}

\begin{table}[h]
\centering
\caption{Computation time of different algorithms under the first tested continuous distribution.\label{tab:exp-time-con}}
\resizebox{\textwidth}{!}{
\begin{tabular}{cccc|cccc|cccc}
	\toprule
	$T$ & Algorithm  & Avg. Regret & Avg. Time(s) & $T$ & Algorithm & Avg. Regret  & Avg. Time(s) & $T$ & Algorithm & Avg. Regret  & Avg. Time(s) \\
	\midrule
	\multirow{3}{*}{$10^3$} & \ref{M1} & $12.37$ & $<0.001$ & \multirow{3}{*}{$10^4$} & \ref{M1} & $38.24$ & $<0.01$ & \multirow{3}{*}{$10^5$} & \ref{M1}  & $123.03$ & $0.063$   \\ 
                            &    \ref{M2} & $4.18$   & $<0.001$   &  & \ref{M2}   & $13.83$    & $<0.01$  & &   \ref{M2}   & $24.00$ &  $0.064$    \\
                            &    \ref{M3} & $3.82$ & $0.95$   &   & \ref{M3}     & $4.12$ & $37.5$   &  &   \ref{M3} & $ 5.91 $ & $4742.9$   \\
    \bottomrule
	\end{tabular}
}
\end{table}

\subsection{Finite support}
\label{subsec:exp-finite}

We generate $\{(c_t, \mathbf{a}_t)\}_{t=1}^T$ from different discrete distributions. The performance of three algorithms is evaluated in terms of $r(\hat{\x}_T) + v(\hat{\x}_T)$. We choose $m \in \{2, 5\}$ and $10$ different $T$ evenly spaced over $[10^3, 10^5]$ on log-scale. All the results are averaged over $100$ independent random trials. To generate a discrete distribution with support size $K$, we first sample $K$ different $\{(c_k, \mathbf{a}_k)\}_{k=1}^K$ from some distribution, then randomly generate a finite probability distribution $\mathbf{p} = (p_1, p_2, \ldots, p_K)$ over $\{(c_k, \mathbf{a}_k)\}_{k=1}^K$. At time $t$, we sample $(c_k, \mathbf{a}_k)$ with probability $p_k$. We generate four discrete distributions as follows: \textbf{1)}. The first distribution takes $m=2, K=5$ and samples $c_k, a_{ki}$ i.i.d. from $\mathcal{U}[0,1]$ and $\mathcal{U}[0,3]$. Each $d_i$ is sampled from  $\mathcal{U}[1/3, 2/3]$. \textbf{2)}. The second distribution takes $m = 5, K=5$ and samples $c_k$ from the folded normal distribution with parameter $\mu = 0$ and $\sigma = 1$; $a_{ki}$ is sampled from the folded normal distribution with $\mu = \sigma = 1$. Each element in $d_i$ is sampled from $\frac{1}{3}(1+|X|)$ with $X\sim \mathcal{N}(0,1)$. \textbf{3)}. The third distribution takes $m=5$, $K=10$ and samples $c_k$ i.i.d. from exponential distribution $\text{exp}(1)$; $a_{ki}$ is sampled from from $\text{exp}(2)$. Each element in $d_i$ is sampled from $(1+|X|)/3$, where $X\sim \text{exp}(1)$. \textbf{4)}. The last distribution takes $m=2$, $K=10$ and samples $c_k$ i.i.d. from $\mathcal{U}[1, 2]$ and $a_{ki}$ from $\Gamma(\alpha, \theta)$ with $(\alpha, \theta) = (2, 3)$. Each element in $d_i$ is sampled from $\mathcal{U}[1/3, 2/3]$.

\begin{figure}[h]
\centering
\includegraphics[scale=0.25]{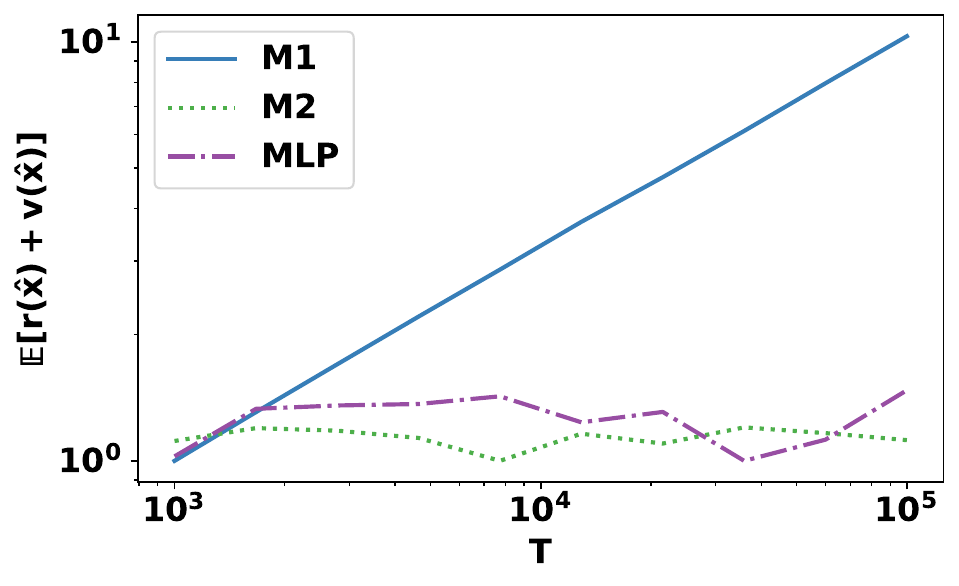}
\includegraphics[scale=0.25]{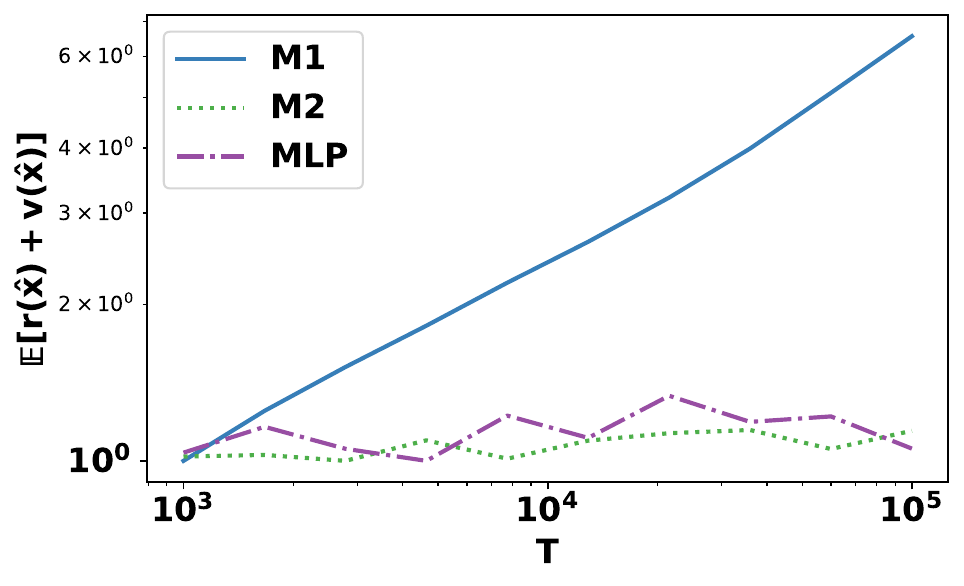}
\includegraphics[scale=0.25]{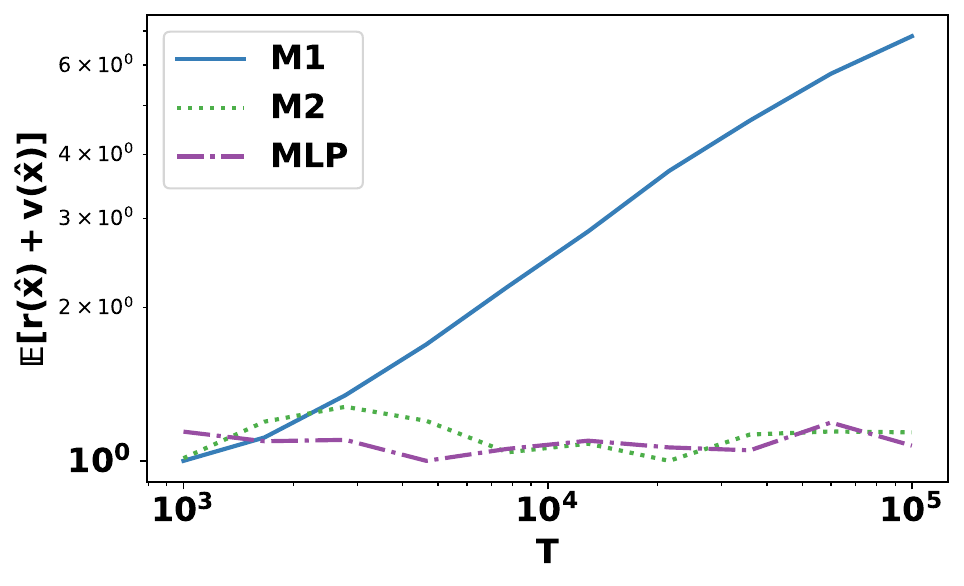}
\includegraphics[scale=0.25]{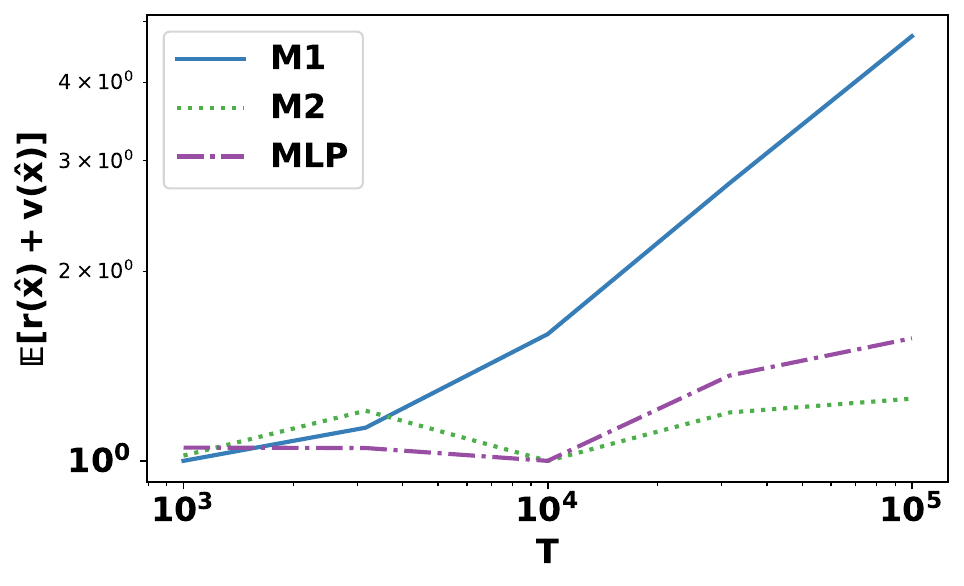}
\caption{\label{fig:exp-dis} Growth of normalized 
$r(\hat{\x}_T)+v(\hat{\x}_T)$ of different algorithms under the finite distributions.}
\end{figure}

For each distribution and algorithm, we plot the normalized regret with respect to $T$. \Cref{fig:exp-dis} indicates that \ref{M2} consistently outperforms \ref{M1} and exhibits $\Ocal(\log T)$ regret. Moreover, \ref{M2} significantly reduces the computation time compared to \ref{M3}. To demonstrate this empirically, we also compare the computation time of \ref{M1}, \ref{M2}, and \ref{M3}. We generate instances according to the fourth distribution with $m=2$ and $T\in \{10^3, 10^4, 10^5\}$. For each $(m, T)$ pair, we average $r+v$ and computation time over $10$ independent trials and summarize the result in \Cref{tab:exp-time-dis}: \ref{M3} greatly reduces computation time compared to \ref{M3} but has comparable regret performance. First-order methods can replace LP-based methods in this case.
\begin{table}[h]
\centering
\caption{Computation time of different algorithms under the last tested finite distribution. \label{tab:exp-time-dis}}
\resizebox{\textwidth}{!}{
\begin{tabular}{cccc|cccc|cccc}
	\toprule
	$T$ & Algorithm & Avg. Regret & Avg. Time(s) & $T$ & Algorithm & Avg. Regret & Avg. Time(s) & $T$ & Algorithm & Avg. Regret & Avg. Time(s) \\
	\midrule
	\multirow{3}{*}{$10^3$} & \ref{M1} & $15.26$ & $<0.001$ & \multirow{3}{*}{$10^4$} & \ref{M1} & $24.39$ & $<0.01$ & \multirow{3}{*}{$10^5$} & \ref{M1} & $71.38$ & $0.080$   \\ 
            &    \ref{M2} & $ 3.61 $  & $<0.001$   &  & \ref{M2}     & $3.00$  & $<0.01$  & &   \ref{M2}   & $3.23$ & $0.084$    \\
            &    \ref{M3}  & $3.04$ & $0.69$   &  & \ref{M3} & $4.03$   & $6.91$   & &   \ref{M3} & $3.62$  & $69.23$   \\
    \bottomrule
	\end{tabular}
}
\end{table}

%% file: sec_conclusion.tex
\section{Conclusion}
In this paper, we propose an online decision-making framework that allows first-order methods to achieve beyond $\mathcal{O}(\sqrt{T})$ regret. We identify that the error bound condition on the dual problem is sufficient for first-order methods to obtain improved regret and design an online learning framework to exploit this condition. We believe that our results provide important new insights for sequential decision-making problems.

%% file: app_errbnd.tex
\section{Proof of Results in \Cref{sec:errbnd}}

\subsection{Auxiliary results}

\begin{lem}[Hoeffding's inequality] \label{lem:hoeffding}
  Let $X_1, \ldots, X_n$ be independent random variables such that $0 \leq X_i
  \leq u$ almost surely. Then for all $\zeta \geq 0$,
  \[ \textstyle \mathbb{P} \{ \tfrac{1}{n} \sum_{i = 1}^n X_i -\mathbb{E} [
     \tfrac{1}{n} \sum_{i = 1}^n X_i ] \geq \zeta \} \leq \exp \{
     - 2 n u^{- 2} \zeta^2 \} . \]
\end{lem}

\begin{lem} \label{lem:app-nondegenlp}
  Consider standard form LP $\min_{\A \x = \tmb, \x \geq \mathbf{0}}  \langle
  \tmc, \x \rangle $ and suppose both primal and dual problems are
  non-degenerate. Then the primal LP solution $\x^{\star}$ is unique and there exists
  $\mu > 0$ such that
  \[ \langle \tmc, \x \rangle - \langle \tmc, \x^{\star}
     \rangle \geq \mu \| \x - \x^{\star} \| \]
     
for all primal feasible $\x \in \{\x: \A\x = \tmb, \x \geq \mathbf{0}\}$.
\end{lem}

\begin{proof}
Denote $\|\x\|_{-\infty} = \min_j {x_j}$. Since both primal and dual problems are non-degenerate, $\x^{\star}$ is
  unique \cite{bertsimas1997introduction}. Denote $(B, N)$ to be the partition of basic and non-basic
  variables, and we have $\x^{\star} = ( \x_B^{\star}, \x_N^{\star} )$,
  where $\x_B^{\star} > \mathbf{0}$ and $\x_N^{\star} = \mathbf{0}$. Similarly, denote $\s$ to
  be the dual slack for $\x$, we can partition $\s^{\star} = (
  \s_B^{\star}, \s_N^{\star} )$ where $\s_B^{\star} = \mathbf{0}$ and
  $\s_N^{\star} > 0$. We have $\A_B \x_B^{\star} = \tmb $ by primal feasibility of $\x^\star$. With dual feasibility,  $\tmc_N = \A_N^{\top}
  \y^{\star} + \s_N^{\star}, \tmc_B = \A_B^{\top} \y^{\star}$ for some
  $\y^{\star}$. Next, consider any feasible LP solution $\x$, and we can write
  \[ \A \x = \A_B \x_B + \A_N \x_N = \tmb = \A_B \x_B^{\star} . \]
  Since $\A_B$ is non-degenerate, taking inverse on both sides gives $\x_B^{\star} = \x_B + \A_B^{- 1} \A_N \x_N$
  and we deduce that
  \begin{align}
    \langle \tmc, \x \rangle - \langle \tmc, \x^{\star}
    \rangle ={} & \langle \tmc_B, \x_B \rangle + \langle
    \tmc_N, \x_N \rangle - \langle \tmc_B, \x^{\star}_B
    \rangle \label{eqn:proof-A-2-1} \\
    ={} & \langle \tmc_B, \x_B \rangle + \langle \tmc_N, \x_N
    \rangle - \langle \tmc_B, \x_B + \A_B^{- 1} \A_N \x_N
    \rangle \label{eqn:proof-A-2-2} \\
    ={} & \langle \tmc_N - \A_N^{\top} \A_B^{- \top} \tmc_B, \x_N
    \rangle \nonumber \\
    ={} & \langle \A_N^{\top} \y^{\star} + \s_N^{\star} - \A_N^{\top}
    \A_B^{- \top} \A_B^{\top} \y^{\star}, \x_N \rangle \label{eqn:proof-A-2-3}\\
    ={} & \langle \s_N^{\star}, \x_N \rangle \geq \|
    \s_N^{\star} \|_{- \infty} \| \x_N \|,  \label{eqn:proof-A-2-4}
  \end{align}
  where \eqref{eqn:proof-A-2-1} uses $\x_N^\star = \mathbf{0}$, \eqref{eqn:proof-A-2-2} plugs in $\x_B^{\star} = \x_B + \A_B^{- 1} \A_N \x_N$, \eqref{eqn:proof-A-2-3} plugs in $\tmc_N = \A_N^{\top}
  \y^{\star} + \s_N^{\star}$ and $\tmc_B = \A_B^{\top} \y^{\star}$, \eqref{eqn:proof-A-2-4} uses the fact that $\s^\star_N > \mathbf{0}$ and $\langle \s_N^{\star}, \x_N \rangle \geq \|\s_N\|_{-\infty} \|\x_N\|_1\geq \|\s_N\|_{-\infty} \|\x_N\|$. Re-arranging the terms,  
  \begin{equation} \label{eqn:proof-A-2-mid}
  	\| \x_N \| \leq \| \s_N^{\star} \|_{- \infty}^{-
  1} ( \langle \tmc, \x \rangle - \langle \tmc,
  \x^{\star} \rangle ).
  \end{equation}
On the other hand, we have
  \begin{align}
    \| \x - \x^{\star} \|^2 ={} & \| \x_B - \x_B^{\star}
    \|^2 + \| \x_N - \x_N^{\star} \|^2 \nonumber\\
    ={} & \| \A_B^{- 1} \A_N \x_N \|^2 + \| \x_N \|^2
    \label{eqn:proof-A-2-5} \\
    ={} & \langle \x_N,  ( \A_N^{\top} \A_B^{- \top} \A_B^{- 1} \A_N +
    \I ) \x_N \rangle \nonumber \\
    \leq{} & \big( \tfrac{\| \A_N \|^2}{\sigma_{\min} ( \A_B
    )^2} + 1 \big) \| \x_N \|^2 \nonumber\\
    \leq{} & \big( \tfrac{\| \A_N \|^2}{\sigma_{\min} ( \A_B
    )^2} + 1 \big) \big( \tfrac{\langle \tmc, \x \rangle -
    \langle \tmc, \x^{\star} \rangle}{\| \s_N^{\star}
    \|_{- \infty}} \big)^2, \label{eqn:proof-A-2-6}
  \end{align}

where \eqref{eqn:proof-A-2-5} again plugs in $\x_B^{\star} = \x_B + \A_B^{- 1} \A_N \x_N$ and $\x_N^\star = \mathbf{0}$; \eqref{eqn:proof-A-2-6} uses the relation \eqref{eqn:proof-A-2-mid}. Taking square-root on both sides gives
  \[ \| \x - \x^{\star} \| \leq ( \tfrac{\| \A_N
     \|^2}{\sigma_{\min} ( \A_B )^2} + 1 )^{1 / 2}
     \tfrac{1}{\| \s_N^{\star} \|_{- \infty}} [ \langle
     \tmc, \x \rangle - \langle \tmc, \x^{\star} \rangle
     ] \leq \tfrac{\| \A_N \| + \sigma_{\min} ( \A_B
     )}{\sigma_{\min} ( \A_B ) \| \s_N^{\star}
     \|_{- \infty}} [ \langle \tmc, \x \rangle -
     \langle \tmc, \x^{\star} \rangle ] \]
  and another re-arrangement of the inequality completes the proof.
\end{proof}

\begin{lem}[Learning algorithm for H\"older growth \cite{xu2017stochastic}] \label{lem:assg}
Consider stochastic optimization problem $\min_{\y \in \mathcal{Y}} f ( \y ) \assign \mathbb{E}_{\xi} [ f ( \y, \xi )
  ]$ with optimal set $\Ycal^\star$ and suppose the following conditions hold:
  \begin{enumerate}[leftmargin=15pt]
    \item there exists some $\y^1 \in \Ycal$ such that $f ( \y^1 ) - f (
    \y^{\star} ) \leq \varepsilon_0$,
    
    \item $\mathcal{Y}^{\star}$ is a nonempty compact set,
    
    \item there exists some constant $G$ such that $\| f' ( \y, \xi
    ) \| \leq G$ for all $\xi$,
    
    \item there exists some constant $\lambda > 0$ and $\theta \in (0, 1]$ such
    that for all $\y \in \mathcal{Y}$
    \[ f ( \y ) - f ( \y^{\star} ) \geq \lambda \cdot
       \dist ( \y, \mathcal{Y}^{\star} )^{1 / \theta} . \]
  \end{enumerate}
Then, there is a first-order method (\Cref{alg:assg}, Algorithm 1, 2, and 4 of \cite{xu2017stochastic}) that
    outputs $f ( \bar{\y}^{T + 1} ) - f ( \y^{\star}
    ) \leq \varepsilon$ after
    \[ T_{\varepsilon} \geq \{ \textstyle \max \{ 9, 1728 \{ \log (
       \tfrac{1}{\delta} ) + \log \lceil \log_2 (
       \tfrac{2\varepsilon_0}{\varepsilon} ) \rceil \}
       \} \tfrac{2^{2(1-\theta)}\lambda^{- 2 \theta} G^2}{\varepsilon^{2 (1 - \theta)}}
       + 1 \} \lceil \log_2 (
       \tfrac{2\varepsilon_0}{\varepsilon} ) \rceil \]
    iterations with probability at least $1 - \delta$.
\end{lem}

\begin{lem}[Last-iterate convergence of stochastic subgradient \cite{liu2023revisiting}] \label{lem:last-iterate}
  Consider stochastic optimization problem $\min_{\y \geq \mathbf{0}} f ( \y
  ) \assign \mathbb{E}_{\xi} [ f ( \y, \xi ) ]$.
  Suppose the following conditions hold:
  \begin{enumerate}[leftmargin=15pt]
    \item There exist $M \geq 0$ such that
    \[ f ( \x ) - f ( \y ) - \langle f' ( \y
       ), \x - \y \rangle \leq M \| \x - \y \| \]
    for all $\x, \y$ and $f' ( \y ) \in \partial f ( \y
    )$,
    
    \item It is possible to compute $\g_{\y}$ such that $\mathbb{E} [
    \g_{\y} ] = f' ( \y )$,
    
    \item $\mathbb{E} [ \| \g_{\y} - f' ( \y ) \|^2
    ] \leq \sigma^2$.
  \end{enumerate}
  Then, the last iterate of the projected subgradient method with stepsize $\alpha$: $\y^{t + 1} = [ \y^t - \alpha \g^t ]_+$
  satisfies
  \[ \mathbb{E} [ f ( \y^{T + 1} ) - f ( \y )
     ] \leq \tfrac{\| \y^1 - \y \|^2}{T \alpha} + 2
     \alpha (M^2 + \sigma^2) (1 + \log T) \]
  for all $\y \geq \mathbf{0}$.
\end{lem}

\Cref{lem:last-iterate} is an application of Theorem C.1, equation (24) of \cite{liu2023revisiting}
with $L = 0, h ( \y ) = 0$ and $\psi(\x) = \frac{1}{2}\|\x\|^2$.

\input{app_assg.tex}

\subsection{Verification of the examples}
\label{app:verify}

\subsubsection{Continuous support}

The result is a direct application of Proposition 2 of
{\cite{li2022online}}.

\subsubsection{Finite support}

Denote $\{ (\xi_k, \bm{\alpha}_k) \}_{k = 1}^K$ to be the support of LP data associated with distribution $\p \in \Rbb^K$. i.e.,
there are $K$ types of customers and customers of type $k$ arrive with
probability $p_k$. We can write the expected dual problem as
\[ \min_{( \y, \bm{\sigma} ) \geq \mathbf{0}} \quad \langle
   \tmd, \y \rangle + \textstyle \textstyle \sum_{k = 1}^K p_i \sigma_i \quad \text{subject
   to} \quad \sigma_i \geq \xi_i - \langle \bm{\alpha}_i, \y
   \rangle, i \in [K] . \]
More compactly, we introduce slack $\bm{\lambda} \in \mathbb{R}^K$ and
define ${{\mathbf{f}}} \assign ( \tmd ; \p ; \mathbf{0} ), \z
\assign ( \y ; \bm{\sigma}; \bm{\lambda} ) \geq \mathbf{0},
\mathbf{Q} \assign ( \A^{\top}, \I, - \I )$. Then, the dual problem
can be written as standard-form.
\begin{equation} \label{eqn:dlp-stform}
	\min_{\z \geq \mathbf{0}} \quad \langle {{\mathbf{f}}}, \z
   \rangle \quad \text{subject to} \quad \mathbf{Q} \z =\bm{\xi}.
\end{equation}
When $\diam (\mathcal{Y}^{\star}) > 0$, the result is an application of weak sharp minima to LP {\cite{burke1993weak}}. When the primal-dual problems are both non-degenerate, $\mathcal{Y}^{\star} = \{\y^{\star} \}$ and applying \tmtextbf{Lemma \ref{lem:app-nondegenlp}}, we get the following
error bound in terms of the LP optimal basis.
\begin{lem}
  Let $(B, N)$ denote the optimal basis partition for \eqref{eqn:dlp-stform} and let $\s_N$ denote the dual slack of primal variables $\z$, then
  \[ \langle {{\mathbf{f}}}, \z \rangle - \langle
     {{\mathbf{f}}}, \z^{\star} \rangle \geq \mu \| \z -
     \z^{\star} \|, \]
  where $\mu = \frac{\sigma_{\min} (\mathbf{Q}_B) \| \s_N \|_{-
  \infty}}{\| \mathbf{Q}_N \| + \sigma_{\min} (\mathbf{Q}_B)}$. Moreover, we
  have $f ( \y ) - f ( \y^{\star} ) \geq \mu \| \y
  - \y^{\star} \|$.
\end{lem}

\begin{proof}
  $\langle {{\mathbf{f}}}, \z \rangle - \langle
  {{\mathbf{f}}}, \z^{\star} \rangle \geq \mu \| \z -
  \z^{\star} \|$ follows from \Cref{lem:app-nondegenlp} applied to the
  compact LP formulation. Next, define $\z_{\y} \assign ( \y ;
  \bm{\sigma}_{\y} ; \bm{\lambda}_{\y} )$ where $\sigma_{\y}
  = [ \bm{\xi}- \textstyle \sum_{k = 1}^K \bm{\alpha}_i y_i ]_+$
  and $\bm{\lambda}_{\y} = \sigma_{\y} -\bm{\xi}+ \textstyle \sum_{k = 1}^K
  \bm{\alpha}_i y_i$. We deduce
  \[ f ( \y ) - f ( \y^{\star} ) = \langle
     \mathbf{f}, \z_{\y} \rangle - \langle \mathbf{f}, \z^{\star}
     \rangle \geq \mu \| \z_{\y} - \z^{\star} \| \geq \mu
     \| \y - \y^{\star} \| \]
  and this completes the proof.
\end{proof}

\subsubsection{General growth}
Given $\y^{\star} \in \arg \min_{\y} f ( \y ) \subseteq
\inte (\mathcal{Y})$, by optimality condition, $\mathbf{0} = \tmd -\mathbb{E}
[ \tma \mathbb{I} \{ c \geq \langle \tma, \y^{\star}
\rangle \} ]$ and
\[ \tmd = \textstyle\int \tma \textstyle\int_{\langle \tma, \y^{\star}
   \rangle}^{\infty} \mathd F ( c| \tma ) \mathd F (
   \tma ), \]
where $F(c, \tma)$ denotes the cdf. of the distribution of $(c, \tma)$. Then we deduce that
\begin{align}
f ( \y ) - f ( \y^{\star} )  ={} & \langle \tmd, \y - \y^{\star} \rangle +\mathbb{E} [
  [ c - \langle \tma, \y \rangle ]_+ - [ c -
  \langle \tma, \y^{\star} \rangle ]_+ ] \nonumber\\
  ={} & \textstyle\int \textstyle\int_{\langle \tma, \y^{\star} \rangle}^{\infty} 
  \langle \tma, \y - \y^{\star} \rangle \mathd F ( c| \tma
  ) \mathd F ( \tma ) + \textstyle\int \textstyle\int_{\langle \tma, \y
  \rangle}^{\langle \tma, \y^{\star} \rangle} \mathd F (
  c| \tma ) \mathd F ( \tma ) \nonumber\\
  ={} & \textstyle\int \textstyle\int_{\langle \tma, \y \rangle}^{\langle \tma,
  \y^{\star} \rangle} \mathbb{I} \{ c \geq v \} \langle \tma, \y -
  \y^{\star} \rangle \mathd v \mathd F ( c, \tma ) + \textstyle\int
  \textstyle\int_{\langle \tma, \y \rangle}^{\langle \tma, \y^{\star}
  \rangle} \mathd F ( c| \tma ) \mathd F ( \tma )
  \nonumber\\
  ={} & \textstyle\int \textstyle\int_{\langle \tma, \y \rangle}^{\langle \tma,
  \y^{\star} \rangle} \mathbb{I} \{ c \geq v \} -\mathbb{I} \{ c
  \geq \langle \tma, \y^{\star} \rangle \} \mathd v \mathd F
  ( c, \tma ) . \nonumber
\end{align}
Next, we invoke the assumptions and
\begin{align}
  & \textstyle\int \textstyle\int_{\langle \tma, \y \rangle}^{\langle \tma,
  \y^{\star} \rangle} \mathbb{I} \{ c \geq v \} -\mathbb{I} \{ c
  \geq \langle \tma, \y^{\star} \rangle \} ~\mathd v  \mathd F
  ( c, \tma ) \nonumber\\
  \geq{} & \tfrac{\lambda_5}{2} \textstyle\int \textstyle\int_{\langle \tma, \y
  \rangle}^{\langle \tma, \y^{\star} \rangle} |
  \langle \tma, \y^{\star} \rangle - v |^p \mathd v \mathd F
  ( \tma ) \nonumber\\
  ={} & \tfrac{\lambda_5}{2 (p + 1)} \mathbb{E} [ | \langle \tma,
  \y - \y^{\star} \rangle |^{p + 1} ] \nonumber\\
  \geq{} & \tfrac{\lambda_5}{2 (p + 1)} \mathbb{E} [ | \langle
  \tma, \y - \y^{\star} \rangle | ]^{p + 1} \label{eqn:proof-A-3-3}\\
  \geq{} & \tfrac{\lambda_4^{p+1}\lambda_5}{2 (p + 1)} \| \y - \y^{\star} \|^{p + 1},
  \nonumber
\end{align}
where \eqref{eqn:proof-A-3-3} uses $p \geq 0$ and that $\Ebb[|X|^{p + 1}] \geq E[|X|]^{p + 1}$. Since $\| \y - \y^{\star} \|^{p + 1} > 0$ for $\y \neq \y^{\star}$,
this completes the proof.

\subsection{Proof of Lemma \ref{lem:alg-conv}}

We verify the conditions in \Cref{lem:assg}.\\

\textbf{Condition 1}.  Take $\y^1 = \mathbf{0} \in \Ycal$. Then
\[ f ( \y^1 ) - f ( \y^{\star} ) \leq f ( \y^1
   ) =\mathbb{E} [ \langle \tmd, \y^1 \rangle + [ c -
   \langle \tma, \y^1 \rangle ]_+ ] =\mathbb{E} [[c]_+]
   \leq \bar{c}, \]
where the first inequality holds since $f ( \y^{\star} ) \geq 0$.\\

\textbf{Condition 2} holds since $\mathcal{Y}^{\star} \subseteq \mathcal{Y}$, $\mathcal{Y}^{\star}$ is closed and $\mathcal{Y}$ is a compact set.\\

\textbf{Condition 3} holds since $\g_{\y} = \tmd - \tma \mathbb{I} \{ c \geq
\langle \tma, \y \rangle \}$ and $\| \g \| \leq
\sqrt{m} (\bar{a} + \bar{d})$. Hence $G = \sqrt{m} (\bar{a} + \bar{d})$.\\

\textbf{Condition 4} holds by the dual error bound condition $f ( \y
) \geq \mu \cdot \dist ( \y, \mathcal{Y}^{\star} )^{\gamma}$
with $\lambda = \mu$ and $\theta = 1 / \gamma$.\\

Now invoke \Cref{lem:assg} and we get that, after
\begin{align}
  T_{\varepsilon} \geq{} & \{ \textstyle \max \{ 9, 1728 \{ \log (
  \tfrac{1}{\delta} ) + \log \lceil \log_2 (
  \tfrac{2\bar{c}}{\varepsilon} ) \rceil \} \}
  \tfrac{2^{2(1-\gamma^{-1})} \mu^{- 2 / \gamma} m (\bar{a} + \bar{d})^2}{\varepsilon^{2 (1 -
  \gamma^{- 1})}} + 1 \} \lceil \log_2 (
  \tfrac{2\bar{c}}{\varepsilon} ) \rceil \nonumber\\
  ={} & \mathcal{O} ( \varepsilon^{- 2 (1 - \gamma^{- 1})} \log (
  \tfrac{1}{\delta} ) \log ( \tfrac{1}{\varepsilon} ) )
  \nonumber
\end{align}

iterations, the algorithm outputs $\bar{\y}^{T + 1}$ such that with probability at least $1 - \delta$,
\[ \mu \cdot \dist ( \bar{\y}^{T + 1},
   \mathcal{Y}^{\star} ) \leq f (
   \bar{\y}^{T + 1} ) - f ( \y^{\star} ) \leq
   \varepsilon \]
and this completes the proof.

\subsection{Proof of Lemma \ref{lem:vr}}

We verify the conditions in \Cref{lem:last-iterate}.\\

\textbf{Condition 1}. Since $f ( \y )$ is convex and has Lipschitz constant $\sqrt{m} (\bar{a} + \bar{d})$, we take $M = 2 \sqrt{m} (\bar{a} + \bar{d})$
and deduce that
\begin{align}
  f ( \x ) - f ( \y ) - \langle f' ( \y
  ), \x - \y \rangle \leq{} & \sqrt{m} (\bar{a} + \bar{d}) \|
  \x - \y \| + \| f' ( \y ) \| \cdot \| \x
  - \y \| \nonumber\\
  \leq{} & 2 \sqrt{m} (\bar{a} + \bar{d}) \| \x - \y \| \nonumber\\
  ={} & M \| \x - \y \| . \nonumber
\end{align}

\textbf{Condition 2} holds in the stochastic i.i.d. input setting.\\

\textbf{Condition 3} holds by taking $\sigma^2 = 4 m (\bar{a} + \bar{d})^2$
and notice that
\[ \mathbb{E} [ \| \g_{\y} - f' ( \y ) \|^2 ]
   \leq 2\mathbb{E} [ \| \g_{\y} \|^2 ] + 2 [
   \| f' ( \y ) \|^2 ] \leq 4 m (\bar{a} +
   \bar{d})^2 . \]
Next, we invoke \Cref{lem:last-iterate} and get last-iterate convergence for $T \geq 3$.
\begin{align}
  \mathbb{E} [ f ( \y^{T + 1} ) - f ( \y ) ]
  \leq{} & \tfrac{\| \y^1 - \y \|^2}{T \alpha} + 2 \alpha (M^2 +
  \sigma^2) (1 + \log T) \nonumber\\
  \leq{} & \tfrac{\| \y^1 - \y \|^2}{T \alpha} + 16 \alpha m (\bar{a} + \bar{d})^2(1 +
  \log T) \label{eqn:proof-3-2-1}\\
  \leq{} & \tfrac{\| \y^1 - \y \|^2}{T \alpha} + 32 \alpha m (\bar{a} + \bar{d})^2 \log
  T, \nonumber
\end{align}
where \eqref{eqn:proof-3-2-1} plugs in $M = 2 \sqrt{m} (\bar{a} + \bar{d})$ and  $\sigma^2 = 4 m (\bar{a} + \bar{d})^2$. Taking $\y = \Pi_{\mathcal{Y}^{\star}} ( \y^1 )$ completes the
proof.

\subsection{Proof of Lemma \ref{lem:dual-conv}}

By definition and the fact that $\Ycal^\star \in \Ycal$, 
\begin{equation}
  \y^{\star} \in \arg \min_{\y \in \Ycal}  ~f ( \y ) \qquad  \text{and} \qquad
  \y_T^{\star} \in \arg \min_{\y \in \Ycal} ~ f_T ( \y ).\nonumber
\end{equation}
According to \ref{A4}, $f ( \y^{\star} ) \leq f ( \y_T^{\star} ) - \mu
\dist ( \y_T^{\star}, \mathcal{Y}^{\star} )^{\gamma}$ and 
\begin{align}
  \mu \cdot \dist ( \y_T^{\star}, \mathcal{Y}^{\star} )^{\gamma} \leq{} &
  f ( \y_T^{\star} ) - f ( \y^{\star} ) \nonumber\\
  ={} & f ( \y_T^{\star} ) - f_T ( \y_T^{\star} ) + f_T
  ( \y_T^{\star} ) - f_T ( \y^{\star} ) + f_T (
  \y^{\star} ) - f ( \y^{\star} ) . \nonumber\\
  \leq{} & f ( \y_T^{\star} ) - f_T ( \y_T^{\star} ) + f_T
  ( \y^{\star} ) - f ( \y^{\star} ), \label{eqn:proof-lem-3-3-1}
\end{align}
where \eqref{eqn:proof-lem-3-3-1} uses $f_T
  ( \y_T^{\star} ) - f_T ( \y^{\star} ) \leq 0 $. Taking expectation and using $\mathbb{E} [ f_T ( \y^{\star} )
] = f ( \y^{\star} )$, we arrive at
\[ \mu \mathbb{E} [ \dist ( \y_T^{\star}, \mathcal{Y}^{\star}
   )^{\gamma} ] \leq \mathbb{E} [ f ( \y_T^{\star}
   ) - f_T ( \y_T^{\star} ) ] \]
and it remains to bound $ f ( \y_T^{\star} ) - f_T (
\y_T^{\star} ) $. For any fixed $\y \in \mathcal{Y}$,
\[ f_T ( \y ) = \tfrac{1}{T} \textstyle \sum_{t = 1}^T \langle \tmd, \y
   \rangle + [ c_t - \langle \tma_t, \y \rangle
   ]_+ \]
and for each $t$, since $\y \geq \mathbf{0}$, 
\begin{align}
  0 \leq{} & \langle \tmd, \y \rangle + [ c_t - \langle
  \tma_t, \y \rangle ]_+ \nonumber\\
  \leq{} & \| \tmd \| \cdot \| \y \| + | c_t | + \|
  \tma_t \| \cdot \| \y \| \nonumber\\
  \leq{} & \sqrt{m} \bar{d} \tfrac{(\bar{c} + \ld)}{\ld} + \bar{c} +
  \sqrt{m} \bar{a} \tfrac{(\bar{c} + \ld)}{\ld} \nonumber\\
  ={} & \sqrt{m} \tfrac{(\bar{a} + \bar{d}) (\bar{c} + \ld)}{\ld} +
  \bar{c} \nonumber
\end{align}
Using \Cref{lem:hoeffding},
\[ \mathbb{P} \{ f ( \y ) - f_T ( \y ) \geq \zeta
   \} \leq \exp \{ - \tfrac{2 \ld^2 T}{( \sqrt{m} (\bar{a} +
   \bar{d}) (\bar{c} + \ld) + \bar{c} \ld )^2} \zeta^2 \}
   . \]
Recall that $\y_T^{\star} \in \mathcal{Y}$ by \eqref{eqn:bounded-dual}, and
we construct an $\varepsilon$-net of $\mathcal{Y}$ as follows:
\[ \mathcal{Y} \subseteq \mathcal{N}_k \assign \bigcup_{\{ j_i \}_{i = 1}^m
   \in \{ 0, \ldots, k \}^m} \{ \y : \| \y - \textstyle \sum_{i = 1}^m
   \tfrac{\bar{c} + \ld}{k \ld} j_i \mathbf{e}_i \|_{\infty}
   \leq \tfrac{\bar{c} + \ld}{k \ld} \}, \]
where we denote the centers of each net as $\mathcal{C}_k$ and $|
\mathcal{C}_k | = (k + 1)^m$.
In each member of the net, we have, by Lipschitz continuity of $f ( \y
)$ and $f_T ( \y )$, that
\[ f ( \y_1 ) - f ( \y_2 ) \leq \sqrt{m} (\bar{a} +
   \bar{d}) \| \y_1 - \y_2 \| \leq m (\bar{a} + \bar{d}) \|
   \y_1 - \y_2 \|_{\infty} \leq \tfrac{m (\bar{a} + \bar{d}) (\bar{c} +
  \ld)}{k \ld} . \]
Next, with union bound,
\begin{align}
  \mathbb{P} \{ \textstyle \max_{\y \in \mathcal{C}_k} f ( \y ) - f_T
  ( \y ) \geq \zeta \} \leq{} & \textstyle \sum_{\z \in \mathcal{C}_k}
  \mathbb{P} \{ f ( \z ) - f_T ( \z ) \geq \zeta
  \} \nonumber\\
  \leq{} & (k + 1)^m \exp \{ - \tfrac{2 \ld^2 T}{( \sqrt{m} (\bar{a}
  + \bar{d}) (\bar{c} + \ld) + \bar{c} \ld )^2} \zeta^2 \}
  . \nonumber
\end{align}
Taking $k = \sqrt{T}$, we have
\begin{align}
  & \mathbb{P} \{ \textstyle \sup_{\y \in \mathcal{Y}} f ( \y ) - f_T
  ( \y ) \leq \zeta + \tfrac{2m (\bar{a} + \bar{d}) (\bar{c} +
\ld)}{\sqrt{T} \ld} \} \nonumber\\
  \geq{} & \mathbb{P} \{ \textstyle \sup_{\y \in \mathcal{Y}} f ( \y ) -
  f_T ( \y ) \leq \zeta + \tfrac{2m (\bar{a} + \bar{d}) (\bar{c} +
  \ld)}{\sqrt{T} \ld} | \textstyle \max_{\y \in \mathcal{C}_k} f ( \y
  ) - f_T ( \y ) \leq \zeta \} \cdot \mathbb{P} \{
  \textstyle \max_{\y \in \mathcal{C}_k} f ( \y ) - f_T ( \y ) \leq
  \zeta \} \nonumber\\
  ={} & \mathbb{P} \{ \textstyle \max_{\y \in \mathcal{C}_k} f ( \y ) - f_T
  ( \y ) \leq \zeta \} \nonumber\\
  \geq{} & 1 - ( \sqrt{T} + 1 )^m \exp \{ - \tfrac{4 \ld^2
  T}{( \sqrt{m} (\bar{a} + \bar{d}) (\bar{c} + \ld) + \bar{c} \ld
  )^2} \zeta^2\}. \nonumber
\end{align}
Taking $\zeta = \sqrt{\tfrac{3 m ( \sqrt{m} (\bar{a} + \bar{d})
(\bar{c} + \ld) + \bar{c} \ld )^2}{4 \ld^2} \tfrac{\log
T}{T}}$ gives
\[ \mathbb{P} \Big\{ \textstyle \sup_{\y \in \mathcal{Y}} f ( \y ) - f_T
   ( \y ) \leq \mathcal{O} \big( \sqrt{\tfrac{\log T}{T}} \big)
   \Big\} \geq 1 - \tfrac{1}{T} \]
and
\[ \mathbb{E} [ \dist ( \y_T^{\star}, \mathcal{Y}^{\star}
   ) ]^{\gamma} \leq \mathbb{E} [ \dist ( \y_T^{\star}, \mathcal{Y}^{\star}
   )^{\gamma} ] \leq \tfrac{1}{\mu}\mathbb{E} [ f ( \y_T^{\star}
   ) - f_T ( \y_T^{\star} ) ] =\mathcal{O} (
   \sqrt{\tfrac{\log T}{T}} ) = o (1) . \]
   
 This completes the proof.

%% file: app_assg.tex
\subsection{Dual learning algorithm}
\label{app:assg}
We include two algorithms in \cite{xu2017stochastic} that can exploit \ref{A4} and achieve the sample complexity in \Cref{lem:alg-conv}. \Cref{alg:assg} is the baseline algorithm and \Cref{alg:assg-r} is its parameter-free variant that adapts to unknown $\lambda$. Note that the algorithm has an explicit projection routine onto $\Ycal' = \{\y \geq \mathbf{0}: \|\y\| \leq \frac{\uc}{\ld}\}$. According to \cite{xu2017stochastic}, given parameters $(\delta, \varepsilon, \varepsilon_0, \gamma, G)$, \Cref{alg:assg} is configured as follows:
\[ K = \lceil \log_2 (
       \tfrac{2\varepsilon_0}{\varepsilon} ) \rceil, \quad D_1 = \tfrac{2^{1-\gamma}\lambda^{-\theta} \varepsilon_0}{\varepsilon^{1-\gamma}}, \quad  t = \max\{9, 1728 \log(\tfrac{K}{\delta}) \} \tfrac{G^2 D_1^2}{\varepsilon_0^2}.\]
       
\begin{algorithm}[h]
\caption{Accelerated Stochastic SubGradient Method (ASSG) \label{alg:assg}}	
\KwIn{Initial point $\y_{0} \in \Ycal' =\{\y \geq \mathbf{0}: \|\y\| \leq \frac{\uc}{\ld}\}$, outer iteration count $K$, inner iteration count $t$, initial error estimate $\varepsilon_0$, initial diameter $D_1$, Lipschitz constant $G$}
Set $\eta_1 = \frac{\varepsilon_0}{3G^2}$

\For{$k = 1, 2, \ldots, K$}{
 Let $\y_1^k = \y_{k-1}$
 
 \For{$\tau = 1, 2, \ldots, t-1$}
 {	
 	$\y_{\tau+1}^{k} = \prod_{\mathcal{\Ycal}\cap \mathcal{B}(\y_{k-1}, D_k)}[\y^k_\tau - \eta_k \g_{\y_\tau^k}]$
 }
 
 Let $\y_k =\frac{1}{t} \sum_{\tau=1}^t \y_\tau^k $
 
 Let $\eta_{k+1}=\frac{1}{2}\eta_k$ and $D_{k+1} = \frac{1}{2}D_k$
 }
 \KwOut{$\y_K$}
\end{algorithm}
 
\begin{algorithm}[h]
\caption{ASSG with Restart (RASSG) \label{alg:assg-r}}	
\KwIn{Initial point $\y^{0} \in \Ycal' = \{\y \geq \mathbf{0}: \|\y\| \leq \frac{\uc}{\ld}\}$, outer iteration count $K$, initial distance $D_1^{(1)}$, inner iteraion count $t_1$, initial error estimate $\varepsilon_0$ and $\omega\in (0, 1]$, error bound parameter $\gamma$, restart round $S$, Lipschitz constant $G$}
Set $\varepsilon_0^{(1)} = \varepsilon_0$, $\eta_1 = \frac{\varepsilon_0}{3G^2}$

\For{$s = 1, 2, \ldots, S$}{
	$\y^{(s)}\leftarrow$ ASSG$(\y^{(s-1)}, K, t_s, D_1^{(s)}, \varepsilon_0^{(s)})$
	
	Let $t_{s+1} = t_s 2^{2(1-\gamma^{-1})}, D_1^{(s+1)} = D_1^{(s)} 2^{1-\gamma^{-1}}$, and $\varepsilon_0^{(s+1)} = \omega \varepsilon_0^{(s)}$
}
\KwOut{$\y^{(S)}$}
\end{algorithm}

%% file: app_framework.tex
\section{Proof of results in \Cref{sec:algo}}

\subsection{Auxiliary results}
\begin{lem}[Bounded dual solution \cite{gao2023solving}] \label{lem:dualconv-1}
  Assume that \ref{A1} to \ref{A3} hold and suppose \Cref{alg:subgrad} with $\alpha_t \equiv \alpha$  starts from $\y^1$ and $\|\y^1\| \leq \tfrac{\bc}{\bld} $, then 
  \begin{align} \label{eqn:bounded-dual-1}
    \| \y^t \| \leq{} & \tfrac{\bc}{\bld} + \tfrac{m ( \ba + \bd
    )^2 \alpha}{2 \bld} + \alpha \sqrt{m} ( \ba + \bd ) = \mathsf{R}, \text{ for all }
    t,
  \end{align}
  almost surely. Moreover, if $\alpha \leq \frac{2 \bld}{3 m ( \ba + \bd )^2} $, then $\y^t \in \Ycal$ for all $t$ almost surely.
\end{lem}

\begin{proof}
The relation \eqref{eqn:bounded-dual-1} follows immediately from
Lemma 5 of {\cite{gao2023solving}}. To see $\y^t \in \Ycal$, we successively deduce, for $\alpha \leq \frac{2 \bld}{3 m ( \ba + \bd )^2}$, that 
\[ \tfrac{m ( \ba + \bd )^2 \alpha}{2 \bld} + \alpha \sqrt{m}
   ( \ba + \bd ) = \tfrac{1}{3} + \tfrac{2 \bld}{3 \sqrt{m} (
   \ba + \bd )} \leq \tfrac{1}{3} + \tfrac{2 ( \ba + \bd )}{3
   \sqrt{m} ( \ba + \bd )} \leq 1 \]
and this completes the proof.	
\end{proof}

\begin{lem}[Subgradient method on strongly convex problems \cite{rakhlin2011making}] \label{lem:sgd-highprob}
Let $\delta \in (0, e^{- 1})$ and assume $T \geq 4$. Suppose $f ( \y
  )$ is $\mu$-strongly convex and $\| \g_{\y} \| \leq G$.
  Then, the subgradient method with stepsize $\alpha_t = 1 / (\mu t)$ satisfies
  \[ \| \y^{T + 1} - \y^{\star} \|^2 \leq \tfrac{624 \log (
     \frac{\log T}{\delta} + 1 ) G^2}{\mu^2 T} \]
  with probability at least $1 - \delta$.
\end{lem}

\begin{lem}[Subgradient method for $\gamma = 2$] \label{lem:gamma-2-conv}
Suppose \ref{A1} to \ref{A3} and \ref{A4} with $\gamma = 2$ hold. Then, the subgradient method with $\alpha_t = 1 / (\mu
  (t + 1))$ outputs $\y^{T + 1}$ such that $\mathbb{E} [ \dist( \y^{T + 1} , \Ycal^\star )^2 ]
     \leq \tfrac{m ( \ba + \bd )^2}{\mu^2
   T}$.
\end{lem}

\begin{proof}
For any $\hat{\y} \in \Ycal^\star$, we deduce
  that
  \begin{align}
    \| \y^{t + 1} - \hat{\y} \|^2 ={} & \| \Pi_{\Ycal}[\y^t - \alpha_t
    \g^t]- \hat{\y} \|^2  \nonumber \\
    \leq{} & \| \y^t - \alpha_t \g^t - \hat{\y}\|^2 \label{proof-lem-2-1} \\
    ={} & \| \y^t - \hat{\y} \|^2 - 2 \alpha_t \langle \y^t -
    \hat{\y}, \g^t \rangle + \alpha_t^2 \| \g^t \|^2,    \nonumber
  \end{align}
where \eqref{proof-lem-2-1} uses the non-expansiveness of the projection operator. Taking $\hat{\y} = \Pi_{\Ycal^\star}[\y^t]$ and using $\| \g^t \|^2 \leq m
  (\bar{a} + \bar{d})^2$, we get
  \[\| \y^{t + 1} - \hat{\y} \|^2\leq \dist( \y^{t} , \Ycal^\star )^2 - 2 \alpha_t \langle \y^t -
    \y, \g^t \rangle + \alpha_t^2 m (\bar{a} + \bar{d})^2.\]
  Since $\mathbb{E}
  [ \g^t ] \in \partial f ( \y^t )$, we have, by
  convexity of $f$, that
  \[ - 2 \langle \y^t - \y^{\star}, \alpha_t \mathbb{E} [ \g^t ]
     \rangle \leq - 2 \alpha_t ( f ( \y^t ) - f (
     \y ) ). \]
Next, we invoke \ref{A4} to get
  \[ f ( \y^t ) - f ( \hat{\y} ) \geq \mu \cdot \dist(\y^t, \Ycal^\star)^2  . \]
Conditioned on history and taking expectation, we have
\begin{align}
  \mathbb{E} [ \dist ( \y^{t + 1}, \mathcal{Y}^{\star}
  )^2 | \y^t] \leq{} & \mathbb{E} [ \| \y^{t + 1} - \hat{\y}
  \|^2 | \y^t ] \nonumber\\
  \leq{} & \dist ( \y^t, \mathcal{Y}^{\star} )^2 - 2 \alpha_t
  \mu \dist ( \y^t, \mathcal{Y}^{\star} )^2 + \alpha_t^2 m
  (\bar{a} + \bar{d})^2 \nonumber\\
  ={} & (1 - 2 \alpha_t \mu) \dist ( \y^t, \mathcal{Y}^{\star}
  )^2 + \alpha_t^2 m (\bar{a} + \bar{d})^2 . \label{eqn:sgd-recur}
\end{align}
With $\alpha_t = \tfrac{1}{\mu (t + 1)}$, we have
\begin{align}
  \mathbb{E} [ \dist( \y^{t + 1} , \Ycal^\star )^2 ] \leq{}
  & (1 - 2\alpha_t \mu) \dist( \y^{t} , \Ycal^\star )^2 + \alpha^2_t m
  ( \ba + \bd )^2 \nonumber\\
  ={} & \tfrac{t - 1}{t + 1} \dist( \y^{t} , \Ycal^\star )^2 + \tfrac{m
  ( \ba + \bd )^2}{\mu^2 (t + 1)^2} . \nonumber
\end{align}
Multiply both sides by $(t + 1)^2$ and we get
\begin{align}
	(t + 1)^2 \mathbb{E} [ \dist( \y^{t + 1} , \Ycal^\star )^2
   ] \leq{} & (t^2 - 1) \dist( \y^{t} , \Ycal^\star )^2 + \tfrac{ m
   ( \ba + \bd )^2}{\mu^2} \\
	4 \mathbb{E} [ \dist( \y^{2}, \Ycal)^2
   ] \leq{} & \tfrac{m
   ( \ba + \bd )^2}{\mu^2}    \label{eqn:extra-term}
\end{align}
Re-arranging the terms, we arrive at
\[ (t + 1)^2 \mathbb{E} [ \dist( \y^{t + 1} , \Ycal^\star )^2
   ] - t^2 \dist( \y^{t} , \Ycal^\star )^2 \leq \tfrac{m (
   \ba + \bd )^2}{\mu^2} . \]
Taking expectation over all the randomness and telescoping from $t = 2$ to $T$, with \eqref{eqn:extra-term} added,
gives
\[ \mathbb{E} [ \dist( \y^{T + 1} , \Ycal^\star )^2 ] \leq
    \tfrac{m (
   \ba + \bd )^2 T}{\mu^2 (T + 1)^2} \leq  \tfrac{m ( \ba + \bd )^2}{\mu^2
   T} \]
and this completes the proof.
\end{proof}

\begin{lem}[Subgradient with constant stepsize] \label{lem:gamma-2-conv-constant} Under the same assumptions as \Cref{lem:gamma-2-conv}, if $\alpha_t \equiv \alpha < 1/(2\mu)$, then
  \[ \mathbb{E} [ \dist (\y^{T + 1}, \Ycal^{\star})^2 ]
     \leq \tfrac{\Delta^2}{\mu \alpha T} + \tfrac{m ( \ba + \bd
     )^2}{\mu} \alpha, \]
  where $\Delta = \dist( \y_1, \Ycal^{\star})$.
\end{lem}

\begin{proof}
Taking $\alpha_t \equiv \alpha < 1/(2\mu)$ and unrolling the recursion from \eqref{eqn:sgd-recur} till $\y^1$, we have
\begin{align}
    \mathbb{E} [ \dist ( \y^{T + 1} - \Ycal^{\star} )^2 ] \leq{} &
    (1 - 2\mu \alpha) \mathbb{E} [ \dist ( \y^{T}, \Ycal^{\star} )^2
    ] + \alpha^2 m (\bar{a} + \bar{d})^2 \nonumber\\
    \leq{} & (1 - 2\mu \alpha)^T \dist ( \y^1, \Ycal^{\star} )^2 + \textstyle \sum_{j =
    0}^{T - 1} \alpha^2 m (\bar{a} + \bar{d})^2 (1 - 2\mu \alpha)^j \nonumber\\
    \leq{} & (1 - 2\mu \alpha)^T \dist ( \y^1, \Ycal^{\star} )^2 + \tfrac{m
    (\bar{a} + \bar{d})^2}{\mu} \alpha \label{proof-lem-2-4} \\
    \leq{} & \tfrac{1}{\mu \alpha T} \dist ( \y^1, \Ycal^{\star} )^2 +
    \tfrac{m (\bar{a} + \bar{d})^2}{\mu} \alpha \label{proof-lem-2-5}\\
    ={} & \tfrac{\Delta^2}{\mu \alpha T} + \tfrac{m (\bar{a} + \bar{d})^2}{\mu}
    \alpha, \nonumber
  \end{align}
  
  where \eqref{proof-lem-2-4} uses the relation $\sum_{j = 0}^{T - 1} (1 - 2\mu \alpha)^j = \tfrac{1 - (1 - 2\mu
  \alpha)^T}{2\mu \alpha} \leq \tfrac{1}{\mu \alpha}$ and \eqref{proof-lem-2-5} is by $(1 - 2\mu \alpha)^T \leq \tfrac{1}{1 + 2\mu \alpha T} \leq \tfrac{1}{\mu \alpha T}$. This completes the proof.
\end{proof}
\subsection{Proof of Lemma \ref{lem:regret}}
By definition of regret, we deduce that
\begin{align} 
  \mathbb{E} [ r ( \hat{\x}_T) ] ={} & \mathbb{E}
  [ \langle \tmc, \x^{\star}_T \rangle - \langle \tmc,
  \hat{\x}_T\rangle ] \nonumber\\
  ={} & \mathbb{E} [ T f_T ( \y_T^{\star} ) - \langle \tmc,
  \hat{\x}_T\rangle ] \label{app:eqn-1} \\
  \leq{} & \mathbb{E} [ T f_T ( \y^{\star} ) - \langle \tmc,
  \hat{\x}_T\rangle ] \label{app:eqn-2} \\
  ={} & T f ( \y^{\star} ) -\mathbb{E} [ \langle \tmc,
  \hat{\x}_T\rangle ] \label{app:eqn-tmp-1}\\
  \leq{} & \mathbb{E} [ \textstyle \sum_{t=1}^T f ( \y^t ) - \langle \tmc,
  \hat{\x}_T\rangle ] \nonumber\\
  ={} & \textstyle \sum_{t = 1}^T \mathbb{E} [ \langle \tmd, \y^t \rangle +
  [ c_t - \langle \tma_t, \y^t \rangle ]_+ - c_t x^t
  ] \label{app:eqn-3}  \\
  ={} & \textstyle \sum_{t = 1}^T \mathbb{E} [ \langle \tmd - \tma_t x^t, \y^t
  \rangle ], \nonumber
\end{align}
where \eqref{app:eqn-1} uses strong duality of LP; \eqref{app:eqn-2} uses the fact $\y^\star$ is a feasible solution and that $\y_T^\star$ is the optimal solution to the sample LP; \eqref{app:eqn-3} uses the definition of $f(\y)$ and that $(c_t, \tma_t)$ are i.i.d. generated. Then we have
\begin{align}
  \| \y^{t + 1} \|^2 - \| \y^t \|^2 ={} & \| [
  \y^t - \alpha ( \tmd - \tma_t x^t ) ]_+ \|^2 -
  \| \y^t \|^2 \nonumber\\
  \leq{} & \| \y^t - \alpha ( \tmd - \tma_t x^t ) \|^2 -
  \| \y^t \|^2 \label{app:eqn-4} \\
  ={} & - 2 \alpha \langle \tmd - \tma_t x^t, \y^t \rangle + \alpha^2
  \| \tmd - \tma_t x^t \|^2 \nonumber\\
  \leq{} & - 2 \alpha \langle \tmd - \tma_t x^t, \y^t \rangle + m
  ( \ba + \bd )^2 \alpha^2, \label{app:eqn-5}
\end{align}
where \eqref{app:eqn-4} uses $\|[\x]_+\|\leq \|\x\|$ and \eqref{app:eqn-5} uses \ref{A2}, \ref{A3}. A simple re-arrangement gives
\begin{equation} \label{app:eqn-6}
	\langle \tmd - \tma_t x^t, \y^t \rangle \leq \tfrac{m ( \ba
   + \bd )^2 \alpha}{2} + \tfrac{\| \y^t \|^2 - \| \y^{t
   + 1} \|^2}{2 \alpha}. 
\end{equation}
Next, we telescope the relation \eqref{app:eqn-6} from $t = 1$ to $T$ and get 
\begin{align}
 \Ebb[ r(\hat{\x}_T) ] ={} & \textstyle \sum_{t = 1}^T \mathbb{E} [ \langle \tmd - \tma_t
  x^t, \y^t \rangle ] \nonumber\\
  \leq{} & \tfrac{m ( \ba + \bd )^2 \alpha}{2} T + \textstyle \sum_{t = 
  1}^T \tfrac{\mathbb{E} [ \| \y^t \|^2 ] -\mathbb{E}
  [ \| \y^{t + 1} \|^2 ]}{2 \alpha} \label{app:eqn-9} \\
  ={} & \tfrac{m ( \ba + \bd )^2 \alpha}{2} T + \tfrac{\mathbb{E}
  [ \| \y^{ 1} \|^2 ] -\mathbb{E} [ \|
  \y^{T + 1} \|^2 ]}{2 \alpha} \nonumber\\
  ={} & \tfrac{m ( \ba + \bd )^2 \alpha}{2} T + \tfrac{\mathbb{E}
  [ \langle \y^{ 1} + \y^{T + 1}, \y^{ 1} - \y^{T + 1}
  \rangle ]}{2 \alpha} \label{app:eqn-10}\\
  \leq{} & \tfrac{m ( \ba + \bd )^2 \alpha}{2} T +
  \tfrac{\mathsf{R}}{\alpha} \mathbb{E} [ \| \y^{ 1} - \y^{T +
  1} \| ] \label{app:eqn-11}\\
  \leq{} & \tfrac{m ( \ba + \bd )^2 \alpha}{2} T +
  \tfrac{\mathsf{R}}{\alpha} \mathbb{E} [ \| \y^{ 1} -
  \y^{\star} \| + \| \y^{T + 1} - \y^{\star} \| ],
  \label{app:eqn-12}
\end{align}

where \eqref{app:eqn-9} again uses relation \eqref{app:eqn-6}; \eqref{app:eqn-11} 
uses the Cauchy's inequality
\[ \langle \y^{ 1} + \y^{ 1}, \y^{ 1} - \y^{T + 1}
   \rangle \leq \| \y^{ 1} + \y^{T + 1} \| \cdot
   \| \y^{ 1} - \y^{T + 1} \| \]
   and almost sure boundedness of iterations derived from \textbf{Lemma \ref{lem:dualconv-1}}:
\begin{align}
   	  \| \y^{ 1} + \y^{T + 1} \| \leq{} & \| \y^{T + 1}
  \| + \| \y^{ 1} \| \leq{} 2 \mathsf{R}. \nonumber
\end{align}   
Finally \eqref{app:eqn-12} is obtained from the triangle inequality
\begin{align}
  \| \y^{ 1} - \y^{T + 1} \| ={} & \| \y^{ 1} -
  \y^{\star} + \y^{\star} - \y^{T + 1} \| \leq \| \y^{ 1} -
  \y^{\star} \| + \| \y^{T + 1} - \y^{\star} \| \nonumber
\end{align}
and this completes the proof.

\subsection{Proof of Lemma \ref{lem:violation}}

For constraint violation, recall that
\[ \mathbb{E} [ v ( \hat{\x}_T ) ] =\mathbb{E} [
   \| [ \A \hat{\x}_T - \tmb ]_+ \| ]
   =\mathbb{E} \big[ \big\| \big[ \textstyle \sum_{t = 1}^T ( \tma_t x^t - \tmd
   ) \big]_+ \big\| \big] \]
and that
\[ \y^{t + 1} = [ \y^{t + 1} - \alpha ( \tmd - \tma_t x^t )
   ]_+ \geq \y^t - \alpha ( \tmd - \tma_t x^t ) . \]
   
A re-arrangement gives
\begin{equation}\label{app:eqn:13}
	\tma_t x^t \leq \tmd + \tfrac{1}{\alpha} ( \y^{t + 1} - \y^t ) .
\end{equation}
and that
\begin{align}
 \textstyle \sum_{t = 1}^T ( \tma_t x^t - \tmd ) 
  \leq{} & \tfrac{1}{\alpha} \textstyle \sum_{t = 1}^T ( \y^{t + 1} - \y^t
  ) \label{app:eqn-14} \\
  ={} & \tfrac{1}{\alpha} ( \y^{T + 1} - \y^1 ) \nonumber
\end{align}
where \eqref{app:eqn-14} uses \eqref{app:eqn:13}. Now, we apply triangle inequality again:  
\begin{align}
  \mathbb{E} [ \| [ \A \hat{\x}_T - \tmb ]_+ \|
  ] \leq{} &  \tfrac{1}{\alpha} \mathbb{E} [ \| \y^{T + 1} - \y^1 \| ] \nonumber\\
  \leq{} & \tfrac{1}{\alpha} \mathbb{E} [
  \| \y^{1} - \y^{\star} \| + \| \y^{T + 1} - \y^{\star}
  \| ], \label{app:eqn:15}
\end{align}
 and this completes the proof.
 
 \subsection{Proof of Lemma \ref{lem:two-phase}}
 
 Similar to the proof of \Cref{lem:regret} and \Cref{lem:violation}, we deduce that
\begin{align}
  \mathbb{E} [ r ( \hat{\x}_T ) ] \leq{} & T f (
  \y^{\star} ) -\mathbb{E} [ \langle \tmc, \hat{\x}_T
  \rangle ] \label{eqn:proof-lem-4-3-1}\\
  ={} & T_e f ( \y^{\star} ) -\mathbb{E} [ \textstyle \sum_{t = 1}^{T_e}
  c_t x^t ] + \textstyle \sum_{t = T_e + 1}^T \mathbb{E} [ f ( \y^{\star}
  ) - c_t x^t ] \nonumber\\
  \leq{} & T_e f ( \y^{\star} ) -\mathbb{E} [ \textstyle \sum_{t = 1}^{T_e}
  c_t x^t ] + \textstyle \sum_{t = T_e + 1}^T \mathbb{E} [ f ( \y^t
  ) - c_t x^t ] \nonumber\\
  ={} & T_e f ( \y^{\star} ) -\mathbb{E} [ \textstyle \sum_{t = 1}^{T_e}
  c_t x^t ] + \textstyle \sum_{t = T_e + 1}^T \mathbb{E} [ \langle \tmd -
  \tma_t x^t, \y^t \rangle ], \nonumber
\end{align}
where \eqref{eqn:proof-lem-4-3-1} is directly obtained from \eqref{app:eqn-tmp-1}. Next, we analyze $\textstyle \sum_{t = T_e + 1}^T \mathbb{E} [ \langle \tmd -
\tma_t x^t, \y^t \rangle ]$. Using \eqref{app:eqn-6},
\[ \langle \tmd - \tma_t x^t, \y^t \rangle \leq \tfrac{m (\bar{a}
   + \bar{d})^2 \alpha}{2} + \tfrac{\| \y^t \|^2 - \| \y^{t +
   1} \|^2}{2 \alpha}, \]
and we deduce that
\begin{align}
  \textstyle \sum_{t = T_e + 1}^T \mathbb{E} [ \langle \tmd - \tma_t x^t, \y^t
  \rangle ] \leq{} & \textstyle \sum_{t = T_e + 1}^T \big[ \tfrac{m (\bar{a} +
  \bar{d})^2 \alpha}{2} + \tfrac{1}{2 \alpha} \mathbb{E} [ \| \y^t
  \|^2 - \| \y^{t + 1} \|^2 ] \big] \nonumber\\
  ={} & \tfrac{m (\bar{a} + \bar{d})^2 \alpha}{2} T_p + \tfrac{1}{2 \alpha}
  \mathbb{E} [ \| \y^{T_e + 1} \|^2 - \| \y^{T + 1}
  \|^2 ] \nonumber\\
  \leq{} & \tfrac{m (\bar{a} + \bar{d})^2 \alpha}{2} T_p +
  \tfrac{\mathsf{R}}{\alpha} \mathbb{E} [ \| \y^{T_e + 1} -
  \y^{\star} \| + \| \y^{T + 1} - \y^{\star} \| ] \label{eqn:proof-4-3-1},
\end{align}
where \eqref{eqn:proof-4-3-1} uses triangle inequality as in \eqref{app:eqn-11}. Next, we consider constraint violation, and we have
\begin{align}
  \mathbb{E} [ v ( \hat{\x}_T ) ] ={} & \mathbb{E}
  [ \| [ \A \hat{\x}_T - \tmb ]_+ \| ]
  \nonumber\\
  ={} & \mathbb{E} [ \| [ \textstyle \sum_{t = 1}^{T_e} ( \tma_t x^t -
  \tmd ) + \textstyle \sum_{t = T_e + 1}^T ( \tma_t x^t - \tmd )
  ]_+ \| ] \nonumber\\
  \leq{} & \mathbb{E} [ \| [ \textstyle \sum_{t = 1}^{T_e} ( \tma_t x^t
  - \tmd ) ]_+ \| ] +\mathbb{E} [ \| [
  \textstyle \sum_{t = T_e + 1}^T ( \tma_t x^t - \tmd ) ]_+ \|
  ], \label{eqn:proof-4-3-2}
\end{align}
where \eqref{eqn:proof-4-3-2} is by $\| [ \x + \y ]_+ \| \leq
\| [ \x ]_+ \| + \| [ \y ]_+ \|$
and we bound
\[ \mathbb{E} [ \| [ \textstyle \sum_{t = T_e + 1}^T ( \tma_t x^t -
   \tmd ) ]_+ \| ] \leq \tfrac{1}{\alpha} \mathbb{E}
   [ \| \y^{T_e + 1} - \y^{\star} \| + \| \y^{T + 1} -
   \y^{\star} \| ]\]
with the same argument as \eqref{app:eqn:15}.
   Putting two relations together and using
\[ V (T_e) = \mathbb{E} [ \| [ \textstyle \sum_{t = 1}^{T_e} (
   \tma_t x^t - \tmd ) ]_+ \| + \textstyle \sum_{t = 1}^{T_e} f (
   \y^{\star} ) - c_t x^t ] . \]
We arrive at
\begin{align}
  & \mathbb{E} [ r ( \hat{\x}_T ) + v (
  \hat{\x}_T ) ] \nonumber\\
  \leq{} & V (T_e) + \tfrac{m (\bar{a} + \bar{d})^2 \alpha}{2} T_p +
  \tfrac{\mathsf{R} + 1}{\alpha} \mathbb{E} [ \| \y^{T_e + 1} -
  \y^{\star} \| +  \| \y^{T + 1} - \y^{\star} \| ]
  \nonumber\\
  \leq{} & V (T_e) + \tfrac{m (\bar{a} + \bar{d})^2 \alpha}{2} T_p +
  \tfrac{\mathsf{R} + 1}{\alpha} \mathbb{E} [ \dist ( \y^{T_e
  + 1}, \mathcal{Y}^{\star} ) + \dist ( \y^{T + 1},
  \mathcal{Y}^{\star} ) + 2 \diam (\mathcal{Y}^{\star}) ], \label{eqn:proof-4-3-3}
\end{align}
where \eqref{eqn:proof-4-3-3} uses
\[ \| \y - \y^{\star} \| = \| \y - \Pi_{\mathcal{Y}^{\star}}
   [ \y ] + \Pi_{\mathcal{Y}^{\star}} [ \y ] -
   \y^{\star} \| \leq \dist ( \y, \mathcal{Y}^{\star} )
   + \diam (\mathcal{Y}^{\star}) \]
for all $\y$ and it remains to analyze $\mathbb{E} [ \dist (
\y^{T_e + 1}, \mathcal{Y}^{\star} ) + \dist ( \y^{T + 1},
\mathcal{Y}^{\star} ) ]$.\\

By \Cref{lem:alg-conv}, we have with probability $1 - 1 / T^{2 \gamma}$ that
\[ \dist ( \y^{T_e + 1}, \mathcal{Y}^{\star} ) \leq \Delta .
\]
Conditioned on the event $\dist ( \y^{T_e + 1}, \mathcal{Y}^{\star}
) \leq \Delta$, we deduce that
\begin{align}
  \mathbb{E} [ \dist ( \y^{T_e + 1}, \mathcal{Y}^{\star}
  ) ] ={} & \mathbb{E} [ \dist ( \y^{T_e + 1},
  \mathcal{Y}^{\star} ) | \dist ( \y^{T_e + 1},
  \mathcal{Y}^{\star} ) \leq \Delta ] \cdot \mathbb{P} \{
  \dist ( \y^{T_e + 1}, \mathcal{Y}^{\star} ) \leq \Delta
  \} \nonumber\\
  & +\mathbb{E} [ \dist ( \y^{T_e + 1}, \mathcal{Y}^{\star}
  ) | \dist ( \y^{T_e + 1}, \mathcal{Y}^{\star} ) >
  \Delta ] \cdot \mathbb{P} \{ \dist ( \y^{T_e + 1},
  \mathcal{Y}^{\star} ) > \Delta \} \nonumber\\
  \leq{} & \Delta + \tfrac{\mathsf{R}}{T^{2 \gamma}}, \label{eqn:proof-4-3-8}\\
  \mathbb{E} [ \dist ( \y^{T_e + 1}, \mathcal{Y}^{\star}
  )^2 ] ={} & \mathbb{E} [ \dist ( \y^{T_e + 1},
  \mathcal{Y}^{\star} )^2 | \dist ( \y^{T_e + 1},
  \mathcal{Y}^{\star} ) \leq \Delta ] \cdot \mathbb{P} \{
  \dist ( \y^{T_e + 1}, \mathcal{Y}^{\star} ) \leq \Delta
  \} \nonumber\\
  & +\mathbb{E} [ \dist ( \y^{T_e + 1}, \mathcal{Y}^{\star}
  )^2 | \dist ( \y^{T_e + 1}, \mathcal{Y}^{\star} ) >
  \Delta ] \cdot \mathbb{P} \{ \dist ( \y^{T_e + 1},
  \mathcal{Y}^{\star} ) > \Delta \} \nonumber\\
  \leq{} & \Delta^2 + \tfrac{\mathsf{R}^2}{T^{2 \gamma}} \label{eqn:proof-4-3-9},
\end{align}
where both \eqref{eqn:proof-4-3-8} and \eqref{eqn:proof-4-3-9} use the fact that $\y^{T_e + 1} \in \Ycal$ imposed by \Cref{alg:two-phase}. Using \Cref{lem:vr}, we have, conditioned on $\y^{T_e + 1}$, that
\begin{align}
  \mathbb{E} [ \dist ( \y^{T + 1}, \mathcal{Y}^{\star} )
  ]^{\gamma} \leq{} & \mathbb{E} [ \dist ( \y^{T + 1},
  \mathcal{Y}^{\star} )^{\gamma} ] \label{eqn:proof-4-3-4} \\
  ={} & \mathbb{E} [ \mathbb{E} [ \dist ( \y^{T + 1},
  \mathcal{Y}^{\star} )^{\gamma} ] | \y^{T_e + 1} ]
  \nonumber\\
  \leq{} & \tfrac{1}{\mu} \mathbb{E} [ \tfrac{1}{\alpha T_p} \dist
  ( \y^{T_e + 1}, \mathcal{Y}^{\star} )^2 + 32 m (\bar{a} + \bar{d})^2 \alpha \log T_p
  ] \label{eqn:proof-4-3-5} \\
  \leq{} & \tfrac{1}{\mu} [ \tfrac{1}{\alpha T_p} \mathbb{E} [
  \dist ( \y^{T_e + 1}, \mathcal{Y}^{\star} )^2 ] + 32 m (\bar{a} + \bar{d})^2
  \alpha \log T ] \label{eqn:proof-4-3-6} \\
  \leq{} & \tfrac{1}{\mu} [ \tfrac{1}{\alpha T_p} ( \Delta^2 +
  \tfrac{\mathsf{R}^2}{T^{2 \gamma}} ) + 32 m (\bar{a} + \bar{d})^2 \alpha \log T ]
  \label{eqn:proof-4-3-7},
\end{align}
where \eqref{eqn:proof-4-3-4} uses $\Ebb[X]^\gamma \leq \Ebb[X^\gamma]$ for nonnegative random variable $X$; \eqref{eqn:proof-4-3-5} invokes \Cref{lem:vr}; \eqref{eqn:proof-4-3-6} uses $T_p \leq T$ and \eqref{eqn:proof-4-3-7} plugs in \eqref{eqn:proof-4-3-9}.
Putting the results together, we get
\begin{align}
  \mathbb{E} [ \dist ( \y^{T + 1}, \mathcal{Y}^{\star} )
  ] \leq{} & ( \tfrac{1}{\mu \alpha T_p} ( \Delta^2 +
  \tfrac{\mathsf{R}^2}{T^{2 \gamma}} ) + \tfrac{32 m (\bar{a} + \bar{d})^2 \alpha}{\mu} \log T
  )^{1 / \gamma} \nonumber\\
  \leq{} & ( \tfrac{1}{\mu} )^{1 / \gamma} \tfrac{1}{\alpha^{1 /
  \gamma} T_p^{1 / \gamma}} [ \Delta^{2 / \gamma} + \tfrac{\mathsf{R}^{2
  / \gamma}}{T^2} ] + ( \tfrac{32 m (\bar{a} + \bar{d})^2}{\mu} )^{1 / \gamma}
  \alpha^{1 / \gamma}  (\log T)^{1 / \gamma} \label{eqn:proof-4-3-10},
\end{align}
where \eqref{eqn:proof-4-3-10} recursively applies $(a + b)^{1/\gamma} \leq a^{1/\gamma} + b^{1/\gamma}$ and we arrive at
\begin{align}
  & \mathbb{E} [ r ( \hat{\x}_T ) + v (
  \hat{\x}_T ) ] \nonumber\\
  \leq{} & V (T_e) + \tfrac{m (\bar{a} + \bar{d})^2}{2} \alpha T_p \nonumber\\
  & + \tfrac{\mathsf{R} + 1}{\alpha} [ \Delta + \tfrac{\mathsf{R}}{T^{2
  \gamma}} + ( \tfrac{1}{\mu} )^{1 / \gamma} \tfrac{1}{\alpha^{1
  / \gamma} T_p^{1 / \gamma}} ( \Delta^{2 / \gamma} +
  \tfrac{\mathsf{R}^{2 / \gamma}}{T^2} ) + ( \tfrac{32 m (\bar{a} + \bar{d})^2}{\mu}
  )^{1 / \gamma} \alpha^{1 / \gamma}  (\log T)^{1 / \gamma} + 2
  \diam (\mathcal{Y}^{\star}) ] \nonumber\\
  ={} & V (T_e) + \tfrac{m (\bar{a} + \bar{d})^2}{2} \alpha T_p + (\mathsf{R} + 1) [
  \tfrac{\Delta}{\alpha} + ( \tfrac{1}{\mu} )^{1 / \gamma} (
  \tfrac{\Delta^{2 / \gamma}}{\alpha^{1 / \gamma + 1} T_p^{1 / \gamma}}
  ) + ( \tfrac{32 m (\bar{a} + \bar{d})^2}{\mu} )^{1 / \gamma} \alpha^{1 / \gamma -
  1}  (\log T)^{1 / \gamma} ] \nonumber\\
  & + (\mathsf{R} + 1) [ \tfrac{\mathsf{R}}{\alpha T^{2 \gamma}} +
  ( \tfrac{1}{\mu} )^{1 / \gamma} \tfrac{\mathsf{R}^{2 /
  \gamma}}{\alpha^{1 / \gamma + 1} T_p^{1 / \gamma} T^2} ] + \tfrac{2
  (\mathsf{R} + 1)}{\alpha} \diam (\mathcal{Y}^{\star}) \nonumber\\
  ={} & V (T_e) +  \mathcal{O} ( \alpha T_p + \tfrac{\Delta}{\alpha} +
  \tfrac{\Delta^{2 / \gamma}}{\alpha^{1 / \gamma + 1} T_p^{1 / \gamma}} +
  \alpha^{1 / \gamma - 1}  (\log T)^{1 / \gamma} + \tfrac{1}{\alpha}
  \diam (\mathcal{Y}^{\star}) + \tfrac{1}{\alpha T^{2 \gamma}} +
  \tfrac{1}{\alpha^{1 / \gamma + 1} T_p^{1 / \gamma} T^2} ) \nonumber
\end{align}

and this completes the proof. Here, the explicit expression of $T_e$ can be
obtained from \Cref{lem:alg-conv}:
\[ T_e = \tfrac{1}{\mu} \{ \max \{ 9, 1728 \{ 2 \gamma \log T +
   \log \lceil \log_2 ( \tfrac{2 \bar{c}}{\Delta^{\gamma}} )
   \rceil \} \} \tfrac{\mu^{- 2 / \gamma} m (\bar{a} +
   \bar{d})^2}{\Delta^{2 (\gamma - 1)}} + 1 \} \lceil \log_2 (
   \tfrac{2 \bar{c}}{\Delta^{\gamma}} ) \rceil . \]
   
\subsection{Proof of Lemma \ref{lem:subgrad-conv}}
Using \Cref{lem:sgd-highprob}, it suffices to verify that the expected dual objective is strongly convex:
\begin{align}
\textstyle  f (y) = \tfrac{1}{2} y +\mathbb{E}_c [[c - y]_+]   =  \tfrac{1}{2} y + \int_y^1 (c - y) \mathd c
  = \tfrac{1}{2} y^2 - \tfrac{1}{2} y + \tfrac{1}{2}.\nonumber
\end{align}
and indeed, $f(y)$ is 1-strongly convex.
   
\subsection{Proof of Lemma \ref{prop:slow_recover_sgd}}
First, we establish the update rule formula for $\mathbb{E}[y^{t+1}]$ in terms of $\mathbb{E}[y^{t}]$. Specifically, we have
\begin{align}
    \mathbb{E}[y^{t+1}]
    ={} &
    \mathbb{E}[[
        y^{t}-\tfrac{1}{t}(\tfrac{1}{2}-\mathbb{I}\{ c_t>y^{t} \})
    ]_{+}] \label{proof-1} \\
    \geq{} &
    \mathbb{E}[
        y^{t}-\tfrac{1}{t}(\tfrac{1}{2}-\mathbb{I}\{ c_t>y^{t} \})
    ]\label{ieq:Ey_update} \\
    \geq{} & 
    \mathbb{E}[
        y^{t}-\tfrac{1}{t}y^{t}+\tfrac{1}{2t} \label{proof-2}
    ]
\end{align}
where \eqref{proof-1} is obtained by the update rule of subgradient, \eqref{ieq:Ey_update} uses Jensen's inequality, and \eqref{proof-2} is obtained by the fact that $c_t$ is independent of $y^{t}$ and it is drawn uniformly from $[0,1]$. Indeed, we have
\[ \mathbb{E} [\mathbb{I} \{ c_t > y^t \}] =\mathbb{E} [\mathbb{E} [\mathbb{I}
   \{ c_t > y^t \} |y^t]] =\mathbb{E} [\textstyle \int_0^1 \mathbb{I} \{ c > y^t \}
   \mathrm{d} c|y^t ] =\mathbb{E} [1 - y^t]. \]

Subtracting $t/2$ from both sides and multiplying both sides the the inequality by $t$, we have
    \begin{align*}
        t(\mathbb{E}[y^{t+1}]-\tfrac{1}{2})
        \geq
        (t-1)(\mathbb{E}[y^{t}]-\tfrac{1}{2}), \quad \text{for all $t=1,\dots,T$.}
    \end{align*}
Next we condition on the value of  $y^{t_0}$ and
    \begin{align}
        \label{ieq:Ey_yt0}
        t(\mathbb{E}[y^{t+1}|y^{t_0}]-\tfrac{1}{2})
        \geq
        (t_0-1)(y^{t_0}-\tfrac{1}{2}).
    \end{align}   
    Thus, given $y^{t_0}>y^\star+\tfrac{1}{\sqrt{T}}=\frac{1}{2}+\tfrac{1}{\sqrt{T}}$ for some $t_0$, we have
    \begin{align}
        \label{ieq:Eyt}
        t(\mathbb{E}[y^{t+1}|y^{t_0}]-\tfrac{1}{2})
        \geq
        (t_0-1)(y^{t_0}-\tfrac{1}{2})
        \geq
        \tfrac{t_0-1}{\sqrt{T}},
    \end{align}      
As a result, when $t_0\geq\tfrac{T}{10}+1$, \eqref{ieq:Eyt} implies
    \begin{align*}
        \mathbb{E}[y^{t+1}|y^{t_0}]
        \geq
        \tfrac{1}{2}+\tfrac{t_0-1}{t\times\sqrt{T}}
        \geq
        \tfrac{1}{2}+\tfrac{1}{10\sqrt{T}},
    \end{align*}
since we assume $t_0 \geq T/10 + 1$. This completes the proof.

\subsection{Proof of Proposition \ref{prop:sgd_badregret}}

Based on \cite{rakhlin2011making}, there exists some universal constant $c > 0$ such that with probability no less than $1-1/T^4$, $|y^t-y^\star| \leq c\log T/\sqrt{T}$ for all $t\geq t_0$, where $y^\star=\tfrac{1}{2}$ and $t_0=\mathcal{O}(\log T)$. Thus, without loss of generality, we assume 
\begin{align}
    \label{cond:hpevent}
    y^t\in[\tfrac{1}{4},\tfrac{3}{4}], \text{ and } y^{t+1}=y^{t}-\tfrac{1}{t}(\tfrac{1}{2}-\mathbb{I}\{c_t>y^t\})
\end{align} 
for all $t\geq t_0$ by setting a new random initialization $y^{t_0}\in[1/4,3/4]$ and ignoring the all decision steps before the $t_0$ step. In the following, we show that {\sgm} using $\Ocal(1/(\mu t))$ stepsize must have $\Omega(T^{1/2})$ regret or constraint violation for any initialization $y^{t_0}$. We first calculate $\mathbb{E}[y^t-\tfrac{1}{2}]$ and $\mathbb{E}[(y^t-\tfrac{1}{2})^2]$ similar to the proof of \Cref{prop:slow_recover_sgd}. Specifically, for $\mathbb{E}[y^t-1/2]$, we have
\begin{align*}
    \mathbb{E}[y^{t+1}|y^t]
    =
    \big(1-\tfrac{1}{t}\big)y^t+\tfrac{1}{2t},
\end{align*}
which implies
\begin{align}
    \label{eqn:yt}
    \mathbb{E}[y^{t+1}-\tfrac{1}{2}|y^{t_0}]
    =
    \tfrac{t_0-1}{t}(y^1-\tfrac{1}{2})+\tfrac{1}{2},
\end{align}
Also, similarly, for $\mathbb{E}[(y^t-1/2)^2]$ we have under assumption \eqref{cond:hpevent}
\begin{align*}
    \mathbb{E}[(y^{t+1}-\tfrac{1}{2})^2|y^t]
    &=
    \mathbb{E}[(y^{t}-\tfrac{1}{t}(\tfrac{1}{2}-\mathbb{I}\{c_t>y^t\})-\tfrac{1}{2})^2|y^t]\\
    &=
    (1-\tfrac{1}{t})^2(y^t-\tfrac{1}{2})^2+\tfrac{1}{4t^2}-\tfrac{1}{t^2}(y^t-\tfrac{1}{2})^2\\
    &\geq
    (1-\tfrac{1}{t})^2(y^t-\tfrac{1}{2})^2+\tfrac{1}{4t^2}-\tfrac{c}{t^3},
\end{align*}
which implies
\begin{align}
    \label{ieq:yt2}
    \mathbb{E}[(y^{t+1}-\tfrac{1}{2})^2|y^t]
    \geq
    \tfrac{(t_0-1)^2}{t^2}(y^{t_0}-\tfrac{1}{2})^2 +\tfrac{1}{4t}-\tfrac{c\log t+t_0}{t^2}.
\end{align}
Combining \eqref{eqn:yt} and \eqref{ieq:yt2}, we then can compute
\begin{align}
    \label{ieq:sumy2}
  &~~~~\textstyle \mathbb{E}[(\sum_{t=t_0}^{T} \mathbb{I}{\{c_t>y^t\}}-\tfrac{T-t_0+1}{2})^2]\\
    &=
   \textstyle \sum_{t=t_0}^{T}\mathbb{E}[(\mathbb{I}{\{c_t>y^t\}}-\tfrac{1}{2})^2]
    +
    2\sum_{t_0\leq i<j\leq T} \mathbb{E}[(\mathbb{I}{\{c_j>y^j\}}-\tfrac{1}{2})(\mathbb{I}{\{c_i>y^i\}}-\tfrac{1}{2})]\nonumber
    \\
    &=
    \textstyle\tfrac{T-t_0}{4}
    +
    2\sum_{t_0\leq i<j\leq T} \mathbb{E}[(\mathbb{I}{\{c_j>y^j\}}-\tfrac{1}{2})(\mathbb{I}{\{c_i>y^i\}}-\tfrac{1}{2})]\nonumber\\
    &=\textstyle
    \tfrac{T-t_0}{4}
    +
    2\sum_{t_0\leq i<j\leq T}\tfrac{i-1}{j-1}\mathbb{E}[(y^i-\tfrac{1}{2})^2]-\tfrac{i-1}{4i(j-1)}\\
    &\geq\textstyle
    \tfrac{T-t_0}{4}-2\sum_{t_0\leq i<j\leq T}\tfrac{c\log T+t_0}{(i-1)^2}\nonumber\\
    &=
    \Omega(T).\nonumber
\end{align}
In addition, since $|y^t-\tfrac{1}{2}|\leq\tfrac{c}{\sqrt{T}}$, by \Cref{lem:hoeffding}, we have with probability no less than $1-\tfrac{1}{T^2}$
\begin{align*}
    \big|\textstyle \sum_{t=t_0}^{T} \mathbb{I}{\{c_t>y^t\}}-\tfrac{T-t_0+1}{2}\big| = \mathcal{O}(\sqrt{T}\log T).
\end{align*}
Consequently, by \eqref{ieq:sumy2}, we have
\begin{align}
    \label{ieq:abs}
    \mathbb{E}\big[\big|\textstyle\sum_{t=t_0}^{T} \mathbb{I}{\{c_t>y^t\}}-\tfrac{T-t_0+1}{2}\big|\big]
    =
    \Omega(\tfrac{\sqrt{T}}{\log T}).
\end{align}
This is the summation of constraint violation and constraint (resource) leftover, and thus, the summation of constraint violation and the regret must be no less than $\Omega(\sqrt{T}/\log T)$.

\subsection{Proof of Theorem \ref{thm:final}}

First note that for sufficiently large $T$, the condition $\alpha_e \leq \frac{2 \bld}{3 m ( \ba +\bd )^2}$ from \Cref{lem:dualconv-1} will be satisfied and all the dual iterates $\{\y^t\}_{t=T_e + 1}^T$ will stay in $\Ycal$ almost surely. When $\diam (\mathcal{Y}^{\star}) = 0$, we consider
\[ \tfrac{1}{\alpha_e} + \alpha_e T_e + \alpha_p T_p +
   \tfrac{\Delta}{\alpha_p} + \tfrac{\Delta^{2 / \gamma}}{\alpha_p^{1 / \gamma
   + 1} T_p^{1 / \gamma}} + \alpha_p^{1 / \gamma - 1}  (\log T)^{1 / \gamma} +
   \tfrac{1}{\alpha_p T^{2 \gamma}} + \tfrac{1}{\alpha_p^{1 / \gamma + 1}
   T_p^{1 / \gamma} T^2} . \]
Since $\alpha_e$ only appears in $\tfrac{1}{\alpha_e} + \alpha_e T_e$, we let
$\alpha_e =\mathcal{O} ( 1/\sqrt{T_e} )$ to optimize the trade-off. Hence, it suffices to consider
\[ \sqrt{T_e} + \alpha_p T_p + \tfrac{\Delta}{\alpha_p} + \tfrac{\Delta^{2 /
   \gamma}}{\alpha_p^{1 / \gamma + 1} T_p^{1 / \gamma}} + \alpha_p^{1 / \gamma
   - 1}  (\log T)^{1 / \gamma} + \tfrac{1}{\alpha_p T^{2 \gamma}} +
   \tfrac{1}{\alpha_p^{1 / \gamma + 1} T_p^{1 / \gamma} T^2} . \]
   Taking $\Delta =\mathcal{O} (T^{- \beta})$ and $\alpha_p =\mathcal{O} (T^{-
\lambda})$ with $(\beta, \lambda) \geq 0$, we have $T_e =\mathcal{O} (T^{2
\beta (\gamma - 1)} \log^2 T)$ according to 
\Cref{lem:alg-conv} and \eqref{eqn:len-explore},  and $\sqrt{T_e} =\mathcal{O} (T^{\beta \gamma
- \beta} \log T)$. Moreover, we have, using $\cong$ to denote equivalence under $\mathcal{O} (\cdot)$ notation, that
\begin{align}
  \alpha_p T_p \cong{} & T^{1 - \lambda} \nonumber\\
  \tfrac{\Delta}{\alpha_p} \cong{} & T^{\lambda - \beta} \nonumber\\
  \tfrac{\Delta^{2 / \gamma}}{\alpha_p^{1 / \gamma + 1} T_p^{1 / \gamma}}
  \cong{} & \tfrac{T^{- 2 \beta / \gamma}}{T^{- \lambda / \gamma - \lambda} T^{1/\gamma}
  (1 - T^{2 \beta (\gamma - 1) - 1} \log^2 T)^{1/\gamma}} = \tfrac{T^{\frac{- 2 \beta
  + \lambda - 1}{\gamma} + \lambda}}{(1 - T^{2 \beta (\gamma - 1) - 1} \log^2
  T)^{1/\gamma}} \nonumber\\
  \alpha_p^{1 / \gamma - 1}  (\log T)^{1 / \gamma} \cong{} & T^{\lambda -
  \lambda / \gamma} (\log T)^{1 / \gamma} \nonumber\\
  \tfrac{1}{\alpha_p T^{2 \gamma}} ={} & \mathcal{O} (1) \nonumber\\
  \tfrac{1}{\alpha_p^{1 / \gamma - 1} T_p^{1 / \gamma} T^2} ={} & \mathcal{O}
  (1) . \nonumber
\end{align}
Suppose $2 \beta (\gamma - 1) - 1 < 0$. Then $\tfrac{\Delta^{2 / \gamma}}{\alpha_p^{1 / \gamma + 1} T_p^{1 / \gamma}}
\cong{} T^{\frac{- 2 \beta + \lambda - 1}{\gamma} + \lambda}$ and
\begin{align}
  & \sqrt{T_e} + \alpha_p T_p + \tfrac{\Delta}{\alpha_p} + \tfrac{\Delta^{2 /
  \gamma}}{\alpha_p^{1 / \gamma + 1} T_p^{1 / \gamma}} + \alpha_p^{1 / \gamma
  - 1}  (\log T)^{1 / \gamma} + \tfrac{1}{\alpha_p T^{2 \gamma}} +
  \tfrac{1}{\alpha_p^{1 / \gamma + 1} T_p^{1 / \gamma} T^2} \nonumber\\
  \cong{} & T^{\beta \gamma - \beta} \log T + T^{1 - \lambda} + T^{\lambda -
  \beta} + T^{\frac{- 2 \beta + \lambda - 1}{\gamma} + \lambda} +
  T^{\lambda - \lambda / \gamma} (\log T)^{1 / \gamma} \nonumber\\
  \lesssim{} & [ T^{\beta \gamma - \beta} + T^{1 - \lambda} + T^{\lambda
  - \beta} + T^{\frac{- 2 \beta + \lambda - 1}{\gamma} + \lambda} +
  T^{\lambda - \lambda / \gamma} ] \log T, \label{eqn:proof-thm-1}
\end{align}
where \eqref{eqn:proof-thm-1} uses $\gamma \geq 1$ and that $(\log T)^{1 / \gamma} \leq
\log T$. To find the optimal trade-off, we solve the following optimization
problem
\begin{eqnarray*}
  \min_{\lambda, \beta} & \max \{ \beta \gamma - \beta, 1 - \lambda,
  \lambda - \beta, \tfrac{- 2 \beta + \lambda - 1}{\gamma} + \lambda,
  \lambda - \tfrac{\lambda}{\gamma} \} & \\
  \text{subject to} & (\lambda, \beta) \geq 0. & 
\end{eqnarray*}
The solution yields $\lambda^{\star} = \frac{\gamma}{2 \gamma - 1}$ and
$\beta^{\star} = \frac{1}{2 \gamma - 1}$ and
\[ 2 \beta^{\star} (\gamma - 1) - 1 = \tfrac{2 \gamma - 2}{2 \gamma - 1} - 1
   = - \tfrac{1}{2 \gamma - 1} < 0 \]
always holds. Hence
\[ T^{\beta^\star \gamma - \beta^\star} + T^{1 - \lambda^\star} + T^{\lambda^\star - \beta^\star} +
   T^{\frac{- 2 \beta^\star + \lambda^\star - 1}{\gamma} + \lambda^\star} + T^{\lambda^\star -
   \lambda^\star / \gamma} =\mathcal{O} ( T^{\frac{\gamma - 1}{2 \gamma - 1}}
   \log T ) \]
and this completes the proof for $\diam (\mathcal{Y}^{\star}) = 0$.\\

Next, consider the case $\diam (\mathcal{Y}^{\star}) > 0$. In this case
we need to consider the trade-off:
\[ \tfrac{1}{\alpha_e} + \alpha_e T_e + \alpha_p T_p +
   \tfrac{\Delta}{\alpha_p} + \tfrac{\Delta^{2 / \gamma}}{\alpha_p^{1 / \gamma
   - 1} T_p^{1 / \gamma}} + \alpha_p^{1 / \gamma - 1}  (\log T)^{1 / \gamma} +
   \tfrac{\diam (\mathcal{Y}^{\star})}{\alpha_p} + \tfrac{1}{\alpha_p
   T^{2 \gamma}} + \tfrac{1}{\alpha_p^{1 / \gamma - 1} T_p^{1 / \gamma} T^2} .
\]
Note that $\tfrac{1}{\alpha_e} + \alpha_e T_e + \alpha_p T_p +
\tfrac{\diam (\mathcal{Y}^{\star})}{\alpha_p} \geq 2 \sqrt{T_e} + 2
\sqrt{T_p \diam (\mathcal{Y}^{\star})}$ and that $T_e + T_p = T$ make it
impossible to achieve better than $\mathcal{O} ( \sqrt{T} )$
regret. Hence, we consider improving the constant associated with $\sqrt{T}$. 

Using $\mathsf{R} = \tfrac{\bar{c}}{\ld} + \mathcal{O} (\max\{\alpha_e, \alpha_p\}))$ and suppose $\alpha_e,\alpha_p$ are of the same order with respect to $T$,
\begin{align}
  \mathbb{E} [ r ( \hat{\x}_T ) + v ( \hat{\x}_T
  ) ] \leq{} & \tfrac{m (\bar{a} + \bar{d})^2}{2} (\alpha_e T_e +
  \alpha_p T_p) + \tfrac{\mathsf{R}}{\alpha_e} + \tfrac{2 (\mathsf{R} +
  1)}{\alpha_p} \diam (\mathcal{Y}^{\star}) \nonumber\\
  & + (\mathsf{R} + 1) [ \tfrac{\Delta}{\alpha_p} + ( \tfrac{1}{
  \mu} )^{1 / \gamma} \tfrac{\Delta^{2 / \gamma}}{\alpha_p^{1 / \gamma +
  1} T_p^{1 / \gamma}} + ( \tfrac{32 m (\bar{a} + \bar{d})^2}{\mu} )^{1 / \gamma}
  \alpha^{1 / \gamma}_p (\log T)^{1 / \gamma} ] +\mathcal{O} (1)
  \nonumber\\
  ={} & \tfrac{m (\bar{a} + \bar{d})^2}{2} (\alpha_e T_e + \alpha_p T_p) +
  \tfrac{\bar{c}}{\ld} \tfrac{1}{\alpha_e} + \tfrac{\bar{c}}{\ld}  \tfrac{2
  \diam (\mathcal{Y}^{\star})}{\alpha_p} \nonumber\\
  & + (\mathsf{R} + 1) [ \tfrac{\Delta}{\alpha_p} + ( \tfrac{1}{
  \mu} )^{1 / \gamma} \tfrac{\Delta^{2 / \gamma}}{\alpha_p^{1 / \gamma +
  1} T_p^{1 / \gamma}} + ( \tfrac{32 m (\bar{a} + \bar{d})^2}{\mu} )^{1 / \gamma}
  \alpha^{1 / \gamma}_p (\log T)^{1 / \gamma} ] +\mathcal{O} (1). \nonumber
\end{align}

Suppose we take $T_e = \theta T$ and $T_p = (1 - \theta) T$ for $\theta \in
(0, 1)$ and we let $\alpha_e = \tfrac{\beta_e}{\sqrt{T_e}} =
\tfrac{\beta_e}{\sqrt{\theta T}}, \alpha_p = \tfrac{\beta_p}{\sqrt{T_p}} =
\tfrac{\beta_p}{\sqrt{(1 - \theta) T}}$. Then, $\Delta = o (1)$ and
\[ (\mathsf{R} + 1) [ \tfrac{\Delta}{\alpha_p} + ( \tfrac{1}{\mu}
   )^{1 / \gamma} \tfrac{\Delta^{2 / \gamma}}{\alpha_p^{1 / \gamma + 1}
   T_p^{1 / \gamma}} + ( \tfrac{32 m (\bar{a} + \bar{d})^2}{\mu} )^{1 / \gamma} \alpha^{1
   / \gamma}_p (\log T)^{1 / \gamma} ] = o ( \sqrt{T} ) . \]
Hence, it suffices to consider
\begin{align}
  & \tfrac{m (\bar{a} + \bar{d})^2}{2} (\alpha_e T_e + \alpha_p T_p) +
  \tfrac{\bar{c}}{\ld} \tfrac{1}{\alpha_e} + \tfrac{\bar{c}}{\ld}  \tfrac{2
  \diam (\mathcal{Y}^{\star})}{\alpha_p} \nonumber\\
  ={} & \tfrac{m (\bar{a} + \bar{d})^2}{2} \beta_e \sqrt{\theta T} + \tfrac{m
  (\bar{a} + \bar{d})^2}{2} \beta_p \sqrt{(1 - \theta) T} +
  \tfrac{\bar{c}}{\ld} \tfrac{\sqrt{\theta T}}{\beta_e} +
  \tfrac{\bar{c}}{\ld}  \tfrac{2 \diam (\mathcal{Y}^{\star})}{\beta_p}
  \sqrt{(1 - \theta) T} \nonumber\\
  ={} & \big[ \tfrac{m (\bar{a} + \bar{d})^2}{2} \beta_e + \tfrac{\bar{c}}{\ld
  \beta_e} \big] \sqrt{\theta T} + \big[ \tfrac{m (\bar{a} +
  \bar{d})^2}{2} \beta_p + \tfrac{2 \diam (\mathcal{Y}^{\star})
  \bar{c}}{\ld \beta_p} \big] \sqrt{(1 - \theta) T.} \nonumber
\end{align}
Taking $\beta_e = \sqrt{\tfrac{2}{m (\bar{a} + \bar{d})^2} \cdot
\tfrac{\bar{c}}{\ld}}$ and $\beta_p = \sqrt{\tfrac{2}{m (\bar{a} +
\bar{d})^2} \cdot \tfrac{2 \diam (\mathcal{Y}^{\star}) \bar{c}}{\ld}}$ to optimize the two trade-offs, we get
\[ \mathbb{E} [ r ( \hat{\x}_T ) + v ( \hat{\x}_T
   ) ] \leq 2 \sqrt{\tfrac{m \bar{c}}{2 \ld}} (\bar{a} + \bar{d})
   \sqrt{\theta T} + 2 \sqrt{2} \sqrt{\tfrac{m \bar{c}}{2 \ld}} (\bar{a} +
   \bar{d}) \sqrt{\diam (\mathcal{Y}^{\star})} \sqrt{(1 - \theta) T} \]
With $\theta = \frac{2 \diam (\mathcal{Y}^{\star})}{2 \diam
(\mathcal{Y}^{\star}) + 1}$, we have
\[ \mathbb{E} [ r ( \hat{\x}_T ) + v ( \hat{\x}_T
   ) ] \leq 4 \sqrt{\tfrac{m \bar{c}}{2 \ld}} \sqrt{\tfrac{2
   \diam (\mathcal{Y}^{\star})}{2 \diam (\mathcal{Y}^{\star}) +
   1}} (\bar{a} + \bar{d}) \sqrt{T}. \]
Since $\diam (\mathcal{Y}^{\star}) \geq 0$, this completes the proof.

\subsection{Removing additional $\log T$ when $\gamma = 2$}

When $\gamma = 2$, the dual error bound condition reduces to quadratic growth, and it is possible to remove the $\log T$ factor in the regret result. Recall that
$\log T$ terms appear when bounding $\mathbb{E} [ \dist (
\y^{T_e + 1}, \mathcal{Y}^{\star} ) ]$ and $\mathbb{E} [
\dist ( \y^{T + 1}, \mathcal{Y}^{\star} )^2 ]$. For $\gamma = 2$, using a tailored analysis, \Cref{lem:gamma-2-conv} guarantees
\[ \mathbb{E} [ \dist ( \y^{T_e + 1}, \mathcal{Y}^{\star}
   ) ] \leq \sqrt{\mathbb{E} [ \dist ( \y^{T_e +
   1}, \mathcal{Y}^{\star} )^2 ]} =\mathcal{O} (
   \tfrac{1}{\sqrt[]{T}} ) . \]
Moreover, using \Cref{lem:gamma-2-conv-constant}, we can directly bound
the expectation
\begin{align}
  \mathbb{E} [ \dist ( \y^{T + 1}, \mathcal{Y}^{\star}
  )^2 ] ={} & \mathbb{E} \big[ \mathbb{E} [ \dist (
  \y^{T + 1}, \mathcal{Y}^{\star} )^2 | \y^{T_e + 1} ]\big ]
  \nonumber\\
  \leq{} & \mathbb{E} \big[ \tfrac{\dist ( \y^{T_e + 1},
  \mathcal{Y}^{\star} )^2}{\mu \alpha T} + \tfrac{m (\bar{a} +
  \bar{d})^2}{\mu} \alpha \big] \nonumber\\
  ={} & \tfrac{\mathbb{E} [ \dist ( \y^{T_e + 1},
  \mathcal{Y}^{\star} )^2 ]}{\mu \alpha T} + \tfrac{m (\bar{a} +
  \bar{d})^2}{\mu} \alpha \nonumber\\
  \leq{} & \tfrac{\Delta^2}{\mu \alpha T} + \tfrac{m (\bar{a} + \bar{d})^2}{\mu}
  \alpha . \nonumber
\end{align}
Therefore, $\log T$ terms can be removed from the analysis.

\subsection{Learning with unknown parameters}

It is possible that $\gamma$ and $\mu$ are unknown in practice. When $\mu$ is unknown, it is possible to run parameter-variants of first-order
methods \Cref{alg:assg-r}, which is slightly more complicated. In terms of $\gamma$, in the finite-support setting, the LP polyhedral error
bound always guarantees $\gamma = 1$. In the continuous support setting, it
suffices to know an upperbound bound on $\gamma$: if \ref{A4} holds for some $\gamma > 0$, then given $\theta > 0$, 
\begin{align}
   f ( \y ) \geq{} & f ( \y^{\star} ) + \mu \dist
  ( \y, \mathcal{Y}^{\star} )^{\gamma} \nonumber\\
  ={} & f ( \y ) - f ( \y^{\star} ) + \mu
  \tfrac{\dist ( \y, \mathcal{Y}^{\star} )^{\gamma +
  \theta}}{\dist ( \y, \mathcal{Y}^{\star} )^{\theta}}
  \nonumber\\
  \geq{} & f ( \y ) - f ( \y^{\star} ) +
  \tfrac{\mu}{\tmop{diam} (\mathcal{Y})^{\theta}} \dist (
  \y, \mathcal{Y}^{\star} )^{\gamma + \theta} . \nonumber
\end{align}

and \ref{A4} also holds for $\gamma' > \gamma$.